\documentclass[pdflatex,sn-mathphys-num]{sn-jnl}

\usepackage{graphicx}%
\usepackage{subfigure}%
\usepackage{multirow}%
\usepackage{amsmath,amssymb,amsfonts}%
\usepackage{amsthm}%
\usepackage{mathrsfs}%
\usepackage[title]{appendix}%
\usepackage{xcolor}%
\usepackage{textcomp}%
\usepackage{manyfoot}%
\usepackage{booktabs}%
\usepackage{algorithm}%
\usepackage{algpseudocode}%
\usepackage{listings}%
\usepackage{amsmath}
\usepackage{makecell}
\usepackage{arydshln}
\usepackage{cases}
\usepackage{circledsteps}
\usepackage{color}
\usepackage{xcolor}

\theoremstyle{thmstyleone}%
\theoremstyle{thmstyletwo}%

\theoremstyle{thmstylethree}%

\DeclareMathAlphabet\mathbfcal{OMS}{cmsy}{b}{n}

\begin{document}


\title{Out-of-Distribution Generalization for Neural Physics Solvers}








\author[1,2]{\fnm{Zhao} \sur{Wei}}\email{weiz@a-star.edu.sg}

\author*[1,2]{\fnm{Chin Chun} \sur{Ooi}}\email{ooicc@a-star.edu.sg}

\author[2]{\fnm{Jian Cheng} \sur{Wong}}\email{wongj@a-star.edu.sg}

\author[3]{\fnm{Abhishek} \sur{Gupta}}\email{abhishekgupta@iitgoa.ac.in}

\author[2]{\fnm{Pao-Hsiung} \sur{Chiu}}\email{chiuph@a-star.edu.sg}

\author[1,2,4]{\fnm{Yew-Soon} \sur{Ong}}\email{asysong@ntu.edu.sg}

\affil*[1]{\orgdiv{Centre for Frontier AI Research (CFAR)}, \orgname{Agency for Science, Technology and Research (A*STAR)}, \orgaddress{\street{1 Fusionopolis Way, \#16-16 Connexis}, \city{Singapore}, \postcode{138632}, 
\country{Republic of Singapore}}}

\affil*[2]{\orgdiv{Institute of High Performance Computing (IHPC)}, \orgname{Agency for Science, Technology and Research (A*STAR)}, \orgaddress{\street{1 Fusionopolis Way, \#16-16 Connexis}, \city{Singapore}, \postcode{138632},
\country{Republic of Singapore}}}

\affil[3]{
\orgdiv{School of Mechanical Sciences}, 
\orgname{Indian Institute of Technology Goa}, 
\orgaddress{
\street{Farmagudi}, 
\city{Ponda}, 
\postcode{403401}, 
\state{Goa}, 
\country{India}}}

\affil[4]{
\orgdiv{College of Computing and Data Science}, 
\orgname{Nanyang Technological University}, \orgaddress{
\street{50 Nanyang Avenue}, 
\city{Singapore}, 
\postcode{639798}, 
\country{Republic of Singapore}}}


\abstract{

Neural physics solvers are increasingly used in scientific discovery, given their potential for rapid \textit{in silico} insights into physical, materials, or biological systems and their long-time evolution. However, poor generalization beyond their training support limits exploration of novel designs and long-time horizon predictions. We introduce NOVA, a route to generalizable neural physics solvers that can provide rapid, accurate solutions to scenarios even under distributional shifts in partial differential equation parameters, geometries and initial conditions. By learning physics-aligned representations from an initial sparse set of scenarios, NOVA consistently achieves 1-2 orders of magnitude lower out-of-distribution errors than data-driven baselines across complex, nonlinear problems including heat transfer, diffusion-reaction and fluid flow. We further showcase NOVA's dual impact on stabilizing long-time dynamical rollouts and improving generative design through application to the simulation of nonlinear Turing systems and fluidic chip optimization. Unlike neural physics solvers that are constrained to retrieval and/or emulation within an \textit{a priori} space, NOVA enables reliable extrapolation beyond known regimes, a key capability given the need for exploration of novel hypothesis spaces in scientific discovery.} 




\maketitle

\section{Introduction}\label{sec1}

AI for Science has achieved multiple breakthroughs in recent years, but there is burgeoning recognition of limitations in current methodologies \cite{szymanski2023autonomous,Claudio2025}. Firstly, scientific discovery inherently prizes novelty, as researchers aim to uncover and understand phenomena, materials, or designs that lie beyond existing knowledge. However, many recent critiques suggest that recent discoveries are extremely similar to materials already within the original dataset, effectively making the AI a ``retrieval" engine as opposed to a true discovery engine \cite{peplow2023robot, cheetham2024artificial}. AI can be rendered ineffective once the search for scientific and/or design novelty brings one to the limits of (or even outside) the training distribution and requires out-of-distribution prediction \cite{karpatne2025ai,zhang2024blending}. Separately, scientific insight in many real-world settings require emulation of long-time transient dynamics, e.g., chemical self-assembly and developmental biology in both natural and/or artificial Turing systems, but this remains a major challenge for AI models given the impact of covariate shift due to error accumulation over long-time rollouts \cite{brunton2024promising,Katiana2024,goswami2022deep}. In both cases, deviations from the training distribution cause severe degradation in their predictive ability, with detrimental consequences in applications like novel design where it can be self-limiting to constrain oneself within the bounds of an \textit{a priori-defined} training distribution, or dynamical simulations where real-time error accumulation is difficult to preempt. Hence, strategies that enable robust zero-shot generalization, i.e., out-of-distribution scenarios where new scenarios may deviate significantly from those observed during training, remain a central challenge in scientific machine learning and a key need in AI for Science \cite{kejriwal2024challenges}.

Two approaches dominate the current paradigm in scientific machine learning but both fail to adequately address the aforementioned challenges. Purely data-driven models (e.g., neural operators) learn well with a large training dataset, but have limited performance when extrapolating beyond the training data distribution \cite{lu2021learning,Azizzadenesheli2024,jiao2025one,cao2024laplace}. Transfer learning and domain adaptation techniques can extend model performance, but their success depends on access to new data and data requirements can still be prohibitive \cite{ding2023parameter,wenzel2022assaying}. Conversely, insights into physical phenomena are often classically derived from solutions to families of partial differential equations (PDEs) spanning different parameters, initial and boundary conditions (ICs/BCs), or geometric domains. Physics-informed machine learning (PIML) circumvents costly data acquisition by embedding physical laws and structure directly into the learning process \cite{RAISSI2019686, brunton2024promising, wei2023select, karniadakis2021physics,liu2025automatic}. However, many standard approaches are typically designed to solve pre-defined scenarios and require substantial retraining when applied to new scenarios \cite{cuomo2022scientific,kadeethum2021framework}. While promising from a data-efficiency perspective, the substantial computational cost associated with training a PIML model for each PDE scenario persists as a significant hurdle due to optimization challenges \cite{krishnapriyan2021characterizing, rathore2024challenges, wong2025evolutionary}. Parameterized or meta-learned PIML models, where one learns over families of PDEs rather than a single PDE, have garnered significant attention, but they remain computationally costly and target learning within the training distribution \cite{JMLR:v25:23-0356, PSAROS2022111121,finn2017model, PENWARDEN2023111912}. 

Hence, there remains a clear gap in the current scientific machine learning paradigm: purely data-driven models perform poorly without sufficient data, while physics-informed models are data-efficient but computationally costly, and neither target out-of-distribution generalization. In this work, we propose Neural architecture discovery with physics for Out-of-distribution, and Versatile Adaptation (NOVA) targeted at data-free and robust prediction beyond (data-limited) training distributions (Fig. \ref{Fig:NOVA}). Rather than optimizing solely for training error under the \textbf{implicit assumption that test scenarios manifest no distributional shift}, NOVA prioritizes physics-guided neural physics solvers whose inductive biases enable versatile, rapid and robust adaptation to new PDE scenarios using the governing physical laws alone. 

The key to NOVA is a paradigm shift in how we define and select for models with physics-aligned inductive biases. NOVA integrates physics-informed zero-shot adaptation directly into the discovery process, identifying inductive biases that are intrinsically aligned with the governing physics. By explicitly enforcing fundamental physical laws, NOVA-derived models also provide physics-compliant and robust predictions for scenarios beyond the available training distribution without requiring any new data. Effective physics-informed adaptation depends critically on models with the right inductive biases, but prior approaches typically focus exclusively on optimizing performance for existing training data or pre-specified tasks \cite{zoph2017neural, pham2018efficient, elsken2019neural, wang2024pinn}. Crucially, the ability to discover models inherently suited for prediction beyond the training distribution remains un-demonstrated. 

Traditional model selection in scientific machine learning is also computationally expensive, as candidate models are typically trained to convergence before evaluation \cite{tack2022meta}. NOVA enables computationally-tractable discovery of models equipped with strong inductive biases for physics and improved capacity for physics-compliant extrapolation through two key innovations. First, we reformulate the physics-informed adaptation as a data-free, convex optimization problem to enable rapid, closed-form solution without iterative and costly retraining. Second, lightweight training is employed in combination with the convexified fast physics-informed adaptation as we exploit the fact that effective feature representations for physics-informed adaptation can emerge early in training (with theoretical analysis in Supplementary Information) \cite{dandi2024two, collins2024provable, ba2022high}. Embedding these two principles reduces computational overhead dramatically, thereby enabling tractable discovery of models with the appropriate inductive biases for physics. 

We validate NOVA across diverse, complex physical systems including nonlinear heat transfer, flow mixing, diffusion-reaction, and Navier-Stokes equations, and spanning tasks with varying PDE parameters, ICs, and geometries. In each instance, NOVA successfully identifies models that can be robustly adapted to new PDE scenarios that fall outside the training distribution without any additional data requirement. These models consistently achieve 5--150$\times$ improvements in predictive accuracy for out-of-distribution tasks compared to purely data-driven methods. 


Given its uniquely favorable trade-off in speed, generalizability, and data requirement, we further demonstrate the potential impact of NOVA on two key problems in AI for Science: 1) Accurate long-time rollouts of complex dynamical systems as exemplified by diffusion-reaction systems characteristic of Turing systems in nano-self-assembly and developmental biology \cite{maini2006turing, newman2007turing, epstein2016reaction, fuseya2021nanoscale}; 2) Contradiction between desire for novel, out-of-distribution design and discovery and brittleness of typical AI models in such scenarios. Specifically, regressor-based guided diffusion has shown great potential in AI for Science, especially in inverse design (e.g., topology optimization and device design) \cite{ho2020denoising, dhariwal2021diffusion, maze2023diffusion, wei2025evolvable}. Flexible conditional generation on pre-trained diffusion models is extremely attractive as foundation models in Science proliferate. However, given the diversity of designs that can be produced through generative models, the data requirements for constructing data-driven regressors can be prohibitive. In contrast, NOVA's robustness to novel, out-of-distribution inputs enable effective integration into guided generative models with minimal high-fidelity data (45 training tasks). Overall, NOVA represents a novel route towards data-efficient, physics-compliant neural physics solvers across diverse physical systems with consistently superior out-of-distribution robustness and predictive accuracy. 

\begin{figure}[htp]
\centering
\vspace*{0mm}
\includegraphics[width=0.98\textwidth]{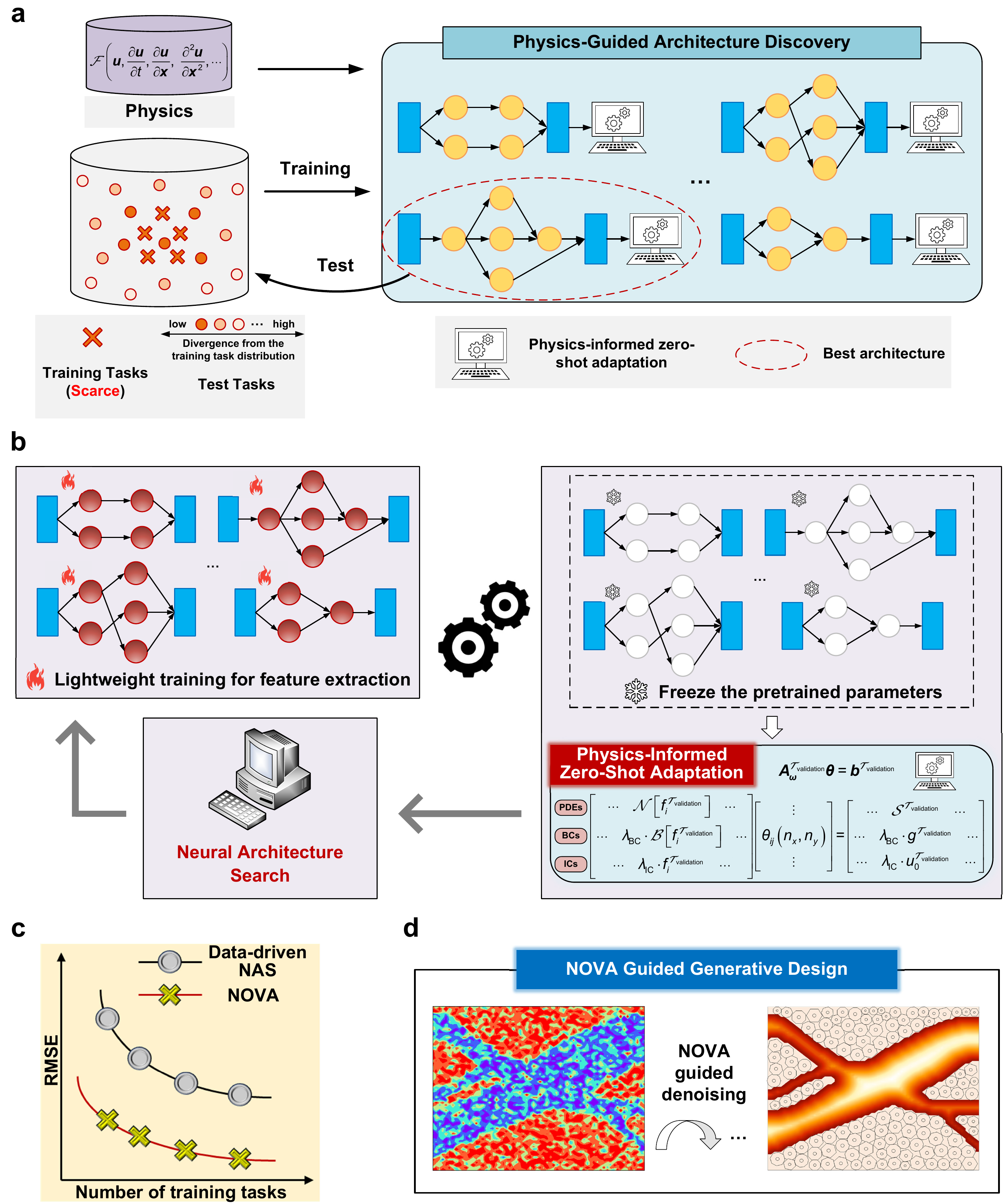}
\vspace*{-0mm}
\caption{Illustration of NOVA. \textbf{a}. Concept of NOVA. In data-scarce settings, NOVA (a physics-guided neural physics solver discovery framework) identifies the optimal neural physics solver by leveraging underlying physics. The selected, more generalizable neural physics solver achieves accurate predictions on test tasks through physics-informed zero-shot adaptation, and provides robust predictions even when the test scenario diverges from the training distribution. \textbf{b}. Algorithmic description of NOVA. During the training and validation process, NOVA automates the design of network architectures (e.g., network operators and connections) on a distribution of PDE scenarios (training data distribution). For each candidate architecture, lightweight training is applied to all kernel parameters except the final layer for effective feature extraction. Subsequently, physics-informed zero-shot adaptation is used for the final layer, ensuring compliance with physical laws during validation. By incorporating physics at this stage, NOVA delivers fast and accurate predictions while being robust to the similarity of the new problem to the (potentially scarce) training distribution. \textbf{c}. Outcome of NOVA. Compared to purely data-driven neural architecture search (NAS) approaches, NOVA demonstrates significantly better generalization performance and is consequently more effective at learning from small datasets. \textbf{d}. NOVA-guided generative design. NOVA can function as a fast, data-light physics-compliant regressor for guided denoising in diffusion models, enabling the generation of designs for fluidic chips with desirable performance characteristics.}
\vspace*{0mm}
\label{Fig:NOVA}
\end{figure}

\section{Results}\label{sec2}

Much of the historical impact from physics originates from our ability to represent diverse physical phenomena mathematically, including as PDEs. Single, problem-specific PDEs are generally categorized into families, including the renowned Navier-Stokes, Poisson and Laplace equations, in recognition of their underlying physics. Generally, these spatio-temporal physical systems are characterized by their PDE parameters, BCs, and ICs as given by

\begin{equation}
\begin{aligned}
& {\rm PDE}: \qquad \quad \mathcal{N}[\boldsymbol{u}(\boldsymbol{x}, t)] = \mathcal{S}(\boldsymbol{x}, t), \qquad & \boldsymbol{x} \in \boldsymbol{\Omega}, t \in (0, t_m] \\
& {\rm BC}: \qquad \quad \quad \mathcal{B}[\boldsymbol{u}(\boldsymbol{x}, t)] = \boldsymbol{g}(\boldsymbol{x}, t), \qquad & \boldsymbol{x} \in \partial{\boldsymbol{\Omega}}, t \in (0, t_m] \\
& {\rm IC}: \qquad \qquad \quad \boldsymbol{u}(\boldsymbol{x}, 0) = \boldsymbol{u}_0(\boldsymbol{x}), \qquad & \boldsymbol{x} \in \boldsymbol{\Omega}, t = 0 
\label{Eq:PINN}
\end{aligned}
\end{equation}

\noindent where $\boldsymbol{x}$ and $t$ are the spatial and temporal coordinates; $\mathcal{N}$ is the differential operator comprising a combination of linear and/or nonlinear derivatives; $\mathcal{S}$ and $\mathcal{B}$ are the source term and boundary operator; $\boldsymbol{u}(\boldsymbol{x}, t)$ is the PDE solution with boundary condition $\boldsymbol{g}(\boldsymbol{x}, t)$ and initial condition $\boldsymbol{u}_0(\boldsymbol{x})$; $\boldsymbol{\Omega}$ and $\partial{\boldsymbol{\Omega}}$ are the respective domain and boundary. There is also usually a simulation time period of interest, $t \in (0, t_m]$, for transient, dynamical physical systems.

In that vein, this work demonstrates the ability of NOVA to discover inductive biases that can efficiently and accurately generalize to new tasks $\mathbfcal{T}_{\rm new}$ (within or outside the original distribution $\mathcal{D}(\mathbfcal{T})$) in data-scarce settings by leveraging known governing physics prescribed as per Eq. (\ref{Eq:PINN}). NOVA's overall framework is detailed in Section \ref{subsec:framework}. 


\subsection{NOVA Performance}\label{subsec:different_parameterizations}

To rigorously evaluate NOVA, we present experiments across representative sources of variation in scientific machine learning scenarios (i.e., differences in governing PDE parameters, ICs, and geometries) spanning multiple complex, nonlinear PDEs:

\begin{itemize}
    \item \textbf{Different ICs (Section \ref{subsec:Initial conditions}).} NOVA is demonstrated on 2D transient diffusion-reaction equations for \textit{different ICs}, i.e., $\boldsymbol{u}_0(\boldsymbol{x})$. Diffusion-reaction equations capture complex processes such as biological pattern formation and are highly sensitive to variations in the initial spatial distribution of activator and inhibitor species. This makes them ideal for testing NOVA’s adaptability to diverse ICs, and ability to emulate long-time horizon system dynamics. 
    \item \textbf{Different geometries (Section \ref{subsec:Geometries}).} NOVA is demonstrated on 2D Navier-Stokes equations with a focus on \textit{changes in geometry}, i.e., different $\boldsymbol{\Omega}$, and their resultant velocity and pressure fields. Geometry-induced variations are critical in many physical systems, especially in the context of generative design, and this provides a rigorous assessment of NOVA’s robustness to different geometric and topological features.
    \item \textbf{Different PDE parameters (Section \ref{subsec:PDEs}).} NOVA is demonstrated for scenarios with \textit{different PDE parameters}, i.e., different $\mathcal{N}$ and $\mathcal{S}$, on the 2D transient nonlinear heat equation and flow mixing equation. The nonlinear heat equation is a parabolic PDE with diffusion-dominated dynamics, while the flow mixing equation is a hyperbolic PDE which models advection-dominated behavior. These two PDEs differ markedly in mathematical structure and physical characteristics, making them well-suited for assessing NOVA’s ability across fundamentally different physical scenarios.

\end{itemize}

Each problem has a defined set of training tasks $\mathbfcal{T}_{\rm training}$, in-distribution test tasks $\mathbfcal{T}_{\rm test\_id}$, and out-of-distribution test tasks $\mathbfcal{T}_{\rm test\_ood}$. $\mathbfcal{T}_{\rm training}$ and $\mathbfcal{T}_{\rm test\_id}$ are chosen from the same distribution while $\mathbfcal{T}_{\rm test\_ood}$ is chosen to be distinct from the training distribution to mimic potential real-world distributional shift. $\mathbfcal{T}_{\rm training}$ is further split into training ($\mathbfcal{T}_{\rm training\_NAS}$) and validation ($\mathbfcal{T}_{\rm validation\_NAS}$) tasks during neural architecture search (NAS). 

Data generation and model training are detailed in Section \ref{subsec:data generation} and Section \ref{subsec:training details}, respectively. For brevity, we refer to the final NOVA-derived model with physics-informed adaptation as NOVA when reporting the results below although NOVA encompasses the complete physics-guided architecture search. While Section \ref{sec2} primarily reports root-mean-squared-error (RMSE), additional metrics are calculated and presented in Supplementary Information, i.e., normalized RMSE (nRMSE), maximum error (max error), RMSE of the conserved value (cRMSE), and RMSE at boundaries (bRMSE) as described in Supplementary Section 8.

\subsection{NOVA robustness to different ICs facilitates long-time horizon prediction for complex dynamical systems} \label{subsec:Initial conditions}

The 2D transient diffusion-reaction equations model a variety of real-world phenomena such as developmental biology and chemical self-assembly \cite{PhysRevE.84.046216}, and are used to validate the effectiveness of NOVA on tasks with varying ICs due to their sensitivity to initialization and highly nonlinear nature. The governing PDEs are given in Eq. (\ref{Eq:2D_DR}) 


\begin{small}
\begin{subequations}
\label{Eq:2D_DR}
\begin{align}
    & \frac{\partial u(x, y, t)}{\partial t} = d_u \Big [\frac{\partial^2 u(x, y, t)}{\partial x^2} + \frac{\partial^2 u(x, y, t)}{\partial y^2}\Big] + u(x, y, t) - u^3(x, y, t) - 0.005 - v(x, y, t), \nonumber \\ 
    & \hspace{7cm} x \in [-1, 1], \; y \in [-1, 1], \; t \in [0, 5] \\[0.5cm]
    & \frac{\partial v(x, y, t)}{\partial t} = d_v \Big[\frac{\partial^2 v(x, y, t)}{\partial x^2} + \frac{\partial^2 v(x, y, t)}{\partial y^2}\Big] + u(x, y, t) - v(x, y, t), \nonumber \\
    & \hspace{7cm} x \in [-1, 1], \; y \in [-1, 1], \; t \in [0, 5]
\end{align}
\end{subequations}
\end{small}

\noindent where $u(x, y, t)$ and $v(x, y, t)$ are the activator and inhibitor and $d_u = 0.001$ and $d_v = 0.005$ are the respective diffusion coefficients for the activator and inhibitor \cite{takamoto2022pdebench}. 

An autoregressive rollout is a common procedure for predicting the transient trajectory of dynamical systems with deep learning models. Under this framework, NOVA predicts the solution at the next time step, $\boldsymbol{u}^{\mathcal{T}_j}(x, y, p\Delta t)$, using the previous time step's predicted solution, $\boldsymbol{u}^{\mathcal{T}_j}(x, y, (p-1)\Delta t)$, and this process is repeated until a specified simulation end-time, $t_m$, is reached. 

\subsubsection{NOVA prediction performance across different ICs}

For this problem, we define in-distribution and out-of-distribution tasks based on the ICs. ICs for $\mathbfcal{T}_{\rm training}$ and $\mathbfcal{T}_{\rm test\_id}$ are generated from standard normal random distribution, i.e., $u(x, y, 0) \sim \mathcal{N}(0,  1)$~\cite{takamoto2022pdebench}. In contrast, ICs for $\mathbfcal{T}_{\rm test\_ood}$ are pre-specified, non-random geometric patterns, thereby introducing a distinct distributional shift. 

We compare NOVA to FNO, U-Net, and MLP-PINN models from \cite{takamoto2022pdebench}, where FNO and U-Net are entirely data-driven models. Fig. \ref{Fig_gp}e and Supplementary Table 2 provide a quantitative comparison of prediction results for NOVA, FNO, U-Net, and MLP-PINN on $\mathbfcal{T}_{\rm test\_id}$ and $\mathbfcal{T}_{\rm test\_ood1}$-$\mathbfcal{T}_{\rm test\_ood3}$.

\textbf{(1) In-distribution results}

Fig. \ref{Fig_gp}a shows that as NAS proceeds, the prediction accuracy of NOVA on $\mathbfcal{T}_{\rm test\_id}$ improves significantly, underscoring the effectiveness of NAS in finding promising architectures. Given the large training dataset ($N = 900$), NOVA and FNO have similar order of magnitude RMSE ($<10^{-2})$, and their predictions are more visually similar. In terms of both global and local error metrics (Supplementary Table 2), NOVA still has a slight advantage over all baseline models. 

\textbf{(2) Out-of-distribution results}

In Fig. \ref{Fig_gp}b-Fig. \ref{Fig_gp}d, we further illustrate the generalization performance of NOVA on three new tasks with ICs out of the training distribution (i.e., $\mathbfcal{T}_{\rm test\_ood1}, \mathbfcal{T}_{\rm test\_ood2}, \mathbfcal{T}_{\rm test\_ood3} \in \mathbfcal{T}_{\rm test\_ood}$). The ground truth (with new ICs) is obtained through an in-house numerical solver (Section \ref{subsubsec:DR}). 

Visually, NOVA still predicts $u$ and $v$ well for $\mathbfcal{T}_{\rm test\_ood1}$-$\mathbfcal{T}_{\rm test\_ood3}$. However, FNO and U-Net extrapolate poorly as these new ICs are significantly different from the training data distribution. Crucially, NOVA's RMSE only slightly increase from 0.0029 to 0.0046, while the RMSE for data-driven FNO and U-Net exhibit orders of magnitude deterioration from $\mathbfcal{T}_{\rm test\_id}$, illustrating the robustness of NOVA to distributional shifts in ICs. 

\textbf{NOVA robustness enhances accuracy for long-time horizon prediction in autoregressive rollout.} Many prior reports have suggested that an autoregressive strategy frequently results in a covariate shift which can adversely affect long-time prediction of dynamical systems in conventional data-driven models \cite{bengio2015scheduled,lamb2016professor}. In contrast, the prediction error of NOVA at each time step is comparatively small, making it well-suited for long-time autoregressive rollout of dynamical systems as compared to FNO and U-Net, as illustrated in Fig. \ref{Fig_gp}f. This highlights the effectiveness of NOVA in leveraging physics for more accurate long-time dynamical emulation.  



\subsubsection{Importance of NOVA's explicit focus on generalizability}

\textbf{Effectiveness of neural architecture}. We examine the impact of an explicit focus on generalization over the conventional goal of minimizing training loss for neural architecture selection through additional ablation experiments. 

Fig. \ref{Fig_gp}g and Fig. \ref{Fig_gp}h present the training loss histograms on $\mathbfcal{T}_{\rm training\_NAS}$ for all 200 candidate architectures evaluated during NAS and the top 20 architectures (based on validation performance on $\mathbfcal{T}_{\rm validation\_NAS}$) obtained during the NAS process for the 2D diffusion-reaction equations, covering both NOVA and NAS-U-Net (described in Section \ref{subsec:training details}). It is evident from Fig. \ref{Fig_gp}g that the best-performing NOVA architectures on $\mathbfcal{T}_{\rm validation\_NAS}$ are not found in the region with minimal training loss during NAS. Instead, they are more common in regions with relatively higher training loss, a sign that overfitting of the training data has been mitigated. In contrast, as shown in Fig. \ref{Fig_gp}h, traditional data-driven training and validation processes (e.g., NAS-U-Net) tend to favor architectures with minimal training loss. This contrast underscores how the incorporation of physics fundamentally transforms the conventional data-driven training and validation paradigm, emphasizing the need for a dedicated NAS algorithm designed to identify generalizable architectures. 

We further select two models from the 200 candidate architectures evaluated during NAS for NOVA: (1) the one with the lowest training loss and (2) another with the median training loss on $\mathbfcal{T}_{\rm training\_NAS}$. Both models are then fine-tuned using the physics-informed zero-shot adaptation for PDE scenarios in $\mathbfcal{T}_{\rm test\_id}$ to illustrate that the improvement originates from the synergy between the physics-aligned architecture and physics-informed adaptation. These models (NOVA-minimal-training-loss and NOVA-median-training-loss in Supplementary Table 3) perform significantly worse than our final NOVA model across all the metrics, underscoring the benefit of NOVA. We illustrate the effectiveness of the physics-informed zero-shot adaptation in Supplementary Section 2.

\begin{figure}[htp]
\centering
\vspace*{0mm}
\includegraphics[width=1\columnwidth]{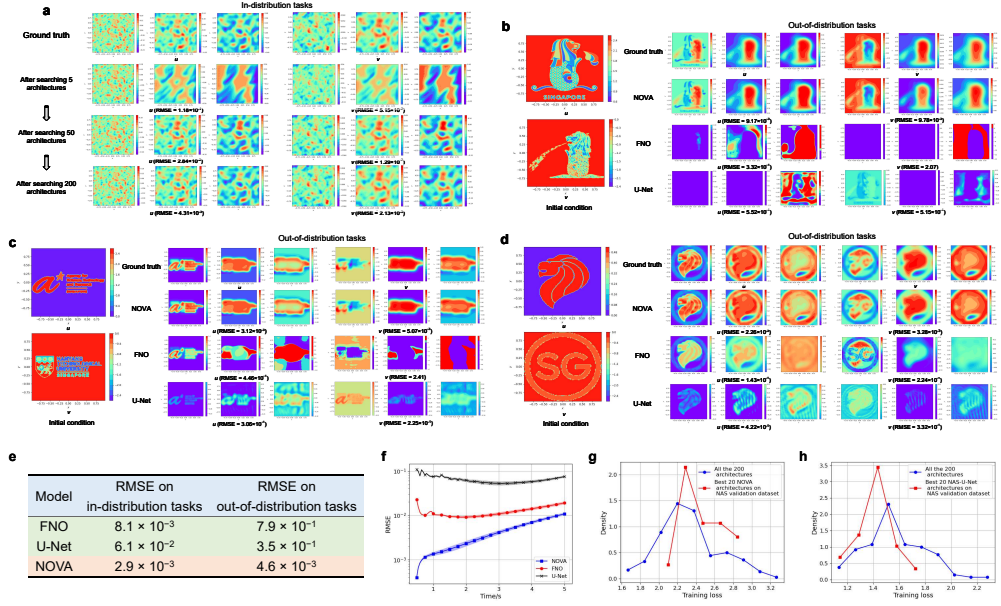}
\vspace*{-0mm}
\caption{Performance of NOVA for the 2D diffusion-reaction equations. \textbf{a}. Prediction results of NOVA on a representative case in $\mathbfcal{T}_{\rm test\_id}$ show improvement after 5, 50, and 200 architectures are evaluated during NAS. \textbf{b-d}. Predictions from NOVA, FNO and U-Net on (b) $\mathbfcal{T}_{\rm test\_ood1}$, (c) $\mathbfcal{T}_{\rm test\_ood2}$, and (d) $\mathbfcal{T}_{\rm test\_ood3}$. \textbf{e}. Mean RMSE on $\mathbfcal{T}_{\rm test\_id}$ and $\mathbfcal{T}_{\rm test\_ood}$ for NOVA, FNO, and U-Net indicates FNO and U-Net performance degrades significantly more than NOVA for out-of-distribution predictions. \textbf{f}. Mean RMSE at each time step across tasks in $\mathbfcal{T}_{\rm test\_id}$ from NOVA, FNO, and U-Net. Shaded area indicates the standard deviation. \textbf{g}. Training loss histogram after five epochs on the training dataset $\mathbfcal{T}_{\rm training\_NAS}$ for all 200 NOVA architectures evaluated during NAS. Also shown are the top 20 architectures, selected based on their validation performance on the validation dataset $\mathbfcal{T}_{\rm validation\_NAS}$ after physics-informed zero-shot adaptation. \textbf{h}. Training loss histogram after twenty epochs on the training dataset $\mathbfcal{T}_{\rm training\_NAS}$ for all 200 NAS-U-Net architectures. The top 20 architectures, identified based on their validation performance on the validation dataset $\mathbfcal{T}_{\rm validation\_NAS}$ during NAS, are also included. For \textbf{a}, \textbf{b}, \textbf{c}, and \textbf{d}, each subfigure for $u$ and $v$ displays predictions at $t = 0.5$, $t = 2.5$, and $t = 5$. RMSE values in brackets indicate an average over the entire simulation time. The pre-trained FNO and U-Net models are obtained from \cite{takamoto2022pdebench}.}
\vspace*{0mm}
\label{Fig_gp}
\end{figure}

\subsection{NOVA robustness to different geometries facilitates data-light novel design} \label{subsec:Geometries}

We next evaluate NOVA for predicting the velocity and pressure fields across different geometries as governed by the challenging 2D steady-state incompressible Navier-Stokes equations. The Navier-Stokes equations are highly nonlinear, making them challenging to solve \cite{CHIU2022114909}, and the computational cost consequently makes it extremely difficult to generate large datasets to cover the design space for pure data-driven models. The governing equations are 


\begin{subequations}
\label{Eq:2D_NS}
\begin{align}
& \frac{\partial u(x, y)}{\partial x} + \frac{\partial v(x, y)}{\partial y} = 0, \nonumber \\
& \hspace{8.5cm} x \in [0, 2], \; y \in [0, 2] \\[0.5cm]
& u(x, y) \frac{\partial u(x, y)}{\partial x} + v(x, y) \frac{\partial u(x, y)}{\partial y} = \frac{1}{\rm Re}\left[\frac{\partial^2 u(x, y)}{\partial x^2} + \frac{\partial^2 u(x, y)}{\partial y^2}\right] - \frac{\partial p(x, y)}{\partial x}, \nonumber \\
& \hspace{8.5cm} x \in [0, 2], \; y \in [0, 2] \\[0.5cm]
& u(x, y) \frac{\partial v(x, y)}{\partial x} + v(x, y) \frac{\partial v(x, y)}{\partial y} = \frac{1}{\rm Re}\left[\frac{\partial^2 v(x, y)}{\partial x^2} + \frac{\partial^2 v(x, y)}{\partial y^2}\right] - \frac{\partial p(x, y)}{\partial y}, \nonumber \\
& \hspace{8.5cm} x \in [0, 2], \; y \in [0, 2]
\end{align}
\end{subequations}

\noindent where $u(x, y)$ and $v(x, y)$ are velocities in the $x$- and $y$- directions; $p(x, y)$ is the pressure; ${\rm Re} = 500$ is the Reynolds number. 


\subsubsection{NOVA prediction performance across different geometries}

For this problem, we distinguish between in-distribution and out-of-distribution scenarios based on the number of inlet channels in the geometry. $\mathbfcal{T}_{\rm training}$ and $\mathbfcal{T}_{\rm test\_id}$ comprise geometries with between one and three inlet channels while $\mathbfcal{T}_{\rm test\_ood}$ comprises geometries with four inlet channels. 

A data-driven U-Net (referenced as NAS-U-Net hereafter) is trained with similar settings as a baseline (described in Section \ref{subsec:training details}). Prediction performance of the physics-guided NOVA and data-driven NAS-U-Net obtained models on both $\mathbfcal{T}_{\rm test\_id}$ and $\mathbfcal{T}_{\rm test\_ood}$ is collected in Fig. \ref{Fig_NS}f and Supplementary Table 4. 

\textbf{(1) In-distribution results}
A representative prediction on $\mathbfcal{T}_{\rm test\_id}$ is provided in Fig. \ref{Fig_NS}a. Visually, both models can capture the similar patterns, although NOVA is quantitatively more accurate. NOVA's RMSE on $\mathbfcal{T}_{\rm test\_id}$ is approximately one order of magnitude better than the data-driven NAS-U-Net (and better across all other metrics in Supplementary Table 4). 

\begin{figure}[htp]
\centering
\vspace*{0mm}
\includegraphics[width=1\textwidth]{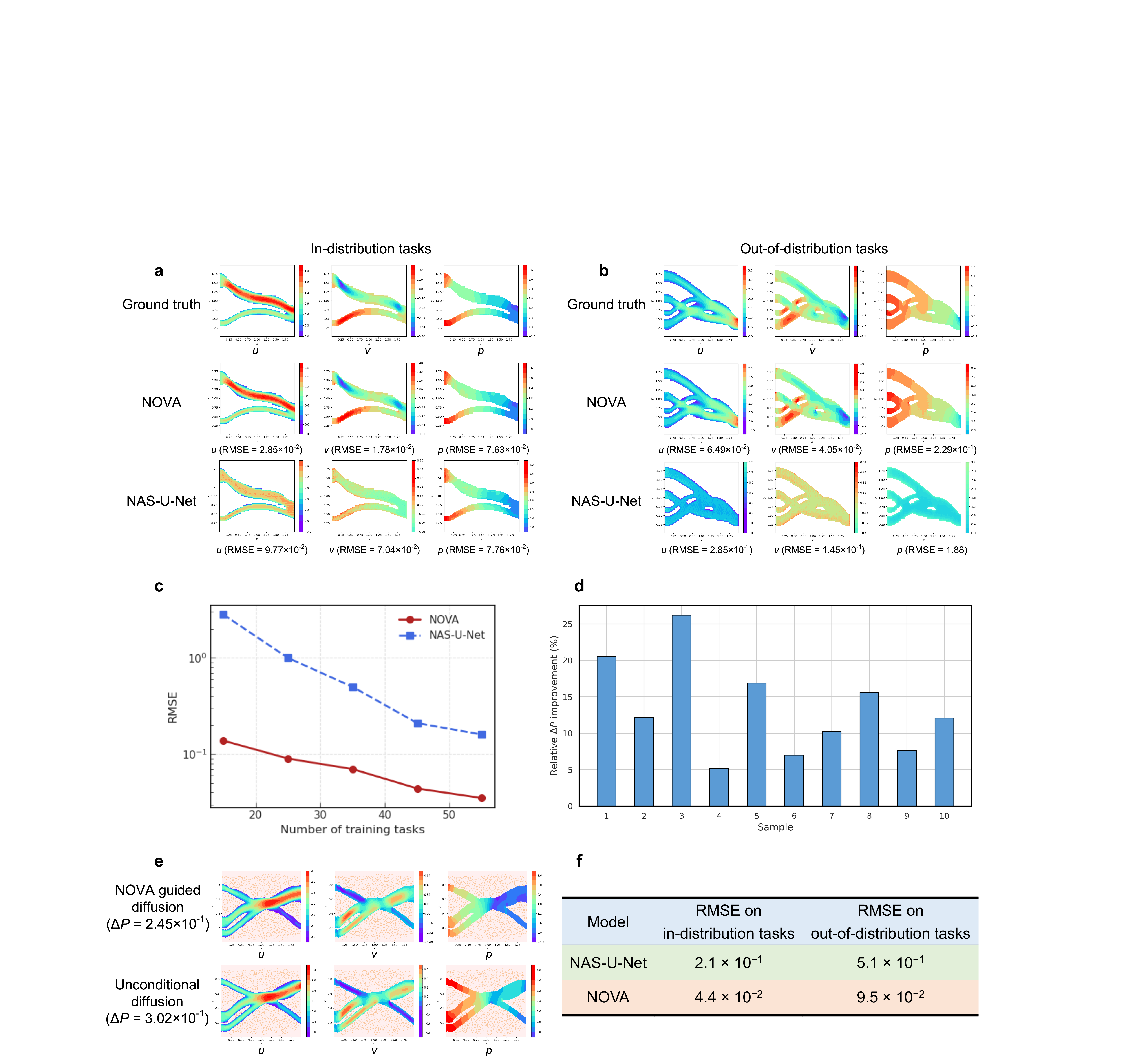}
\vspace*{-0mm}
\caption{Performance of NOVA for the 2D Navier-Stokes equations. \textbf{a-b}. Representative prediction results from NOVA and NAS-U-Net on (a) $\mathbfcal{T}_{\rm test\_id}$ and (b) $\mathbfcal{T}_{\rm test\_ood}$. The three images in each subfigure represent the predictions of $u$, $v$, and $p$ while the RMSE in brackets indicates the corresponding errors. \textbf{c}. RMSE for NOVA and NAS-U-Net on $\mathbfcal{T}_{\rm test\_id}$ when trained with 15, 25, 35, 45, and 55 training samples. \textbf{d}. Relative improvement in $\Delta P$ after guided generation with NOVA. Relative improvement is calculated as the difference between the $\Delta P$ for a design obtained by unconditional diffusion and NOVA-guided diffusion, normalized by the former's $\Delta P$. \textbf{e}. CFD-derived velocity and pressure fields for a sample fluidic chip design generated through diffusion with and without NOVA guidance. \textbf{f}. Mean RMSE on $\mathbfcal{T}_{\rm test\_id}$ and $\mathbfcal{T}_{\rm test\_ood}$ for NOVA and NAS-U-Net trained on 45 geometries. The $\Delta P$ values in \textbf{d} and \textbf{e} are obtained using a CFD solver on generated geometries.} 
\vspace*{0mm}
\label{Fig_NS}
\end{figure}

To further illustrate the generalization performance of NOVA under data-sparse scenarios, NOVA and NAS-U-Net are pre-trained on small datasets comprising 15, 25, 35, 45, or 55 different geometries. Test results on $\mathbfcal{T}_{\rm test\_id}$ are plotted in Fig. \ref{Fig_NS}c. Under such data scarcity, NOVA can still achieve RMSE $<0.1$, consistently outperforming NAS-U-Net through physics. Unsurprisingly, both models are still improving with more training data. 


\textbf{(2) Out-of-distribution results}

Using geometries with four inlet channels as a form of out-of-distribution scenarios ($\mathbfcal{T}_{\rm test\_ood}$), we further observe that NOVA is better able to provide accurate predictions and significantly outperforms the data-driven NAS-U-Net, both qualitatively as in Fig. \ref{Fig_NS}b, and quantitatively in Fig. \ref{Fig_NS}f, and Supplementary Table 4. In particular, Fig. \ref{Fig_NS}b clearly shows how the data-driven model is completely unable to provide a meaningful guess under such distributional shifts and these results underscore NOVA's improved generalization capability, even for complex, nonlinear physics. 

\subsubsection{NOVA-guided generative design}\label{subsubsec:generative design}

Fluidic chips are being actively developed for applications like organ-on-chip for therapeutic developments, and ensuring optimal performance through metrics such as perfusion efficiency and low operating pressures (to avoid channel delamination and/or cell damage) is a key part of the design process \cite{bhatia2014microfluidic, low2021organs, emmerich2024automated}. In this work, we apply the NOVA model within a guided diffusion model \cite{ho2020denoising} as a demonstration of the utility of data-light NOVA for novel design. The NOVA model is pre-trained on 45 different tasks and model details are in Section \ref{subsec:guided diffusion}. 

Guided generation with NOVA utilizes a conditional denoising process as described in Eq. (\ref{denoising process}) and Eq. (\ref{conditional denoising process}). During each conditional denoising step, the NOVA model is provided the putative fluid geometry, and the model's predicted pressure drop across the geometry ($\Delta P$) is used to guide the denoising process (i.e., geometry generation) such that $\Delta P$ is reduced/minimized. 

We generate 10 samples with and without NOVA guidance. As shown in Fig. \ref{Fig_NS}d, designs generated with NOVA guidance consistently achieve lower $\Delta P$ than without. A Wilcoxon signed-rank test confirms that this improvement is statistically significant, with an average improvement of $\approx15\%$ ($p < 0.001$). In contrast, using the NAS-U-Net for guidance shows no statistically significant improvement ($p \approx0.99$). This demonstrates NOVA's ability to steer the generative process towards desirable outcomes. As illustrated in Fig. \ref{Fig_NS}e, NOVA guidance widens the outlet channel, thereby reducing $\Delta P$. 


\subsection{NOVA robustness to different PDE parameters as a potential path to more generalizable PDE solvers} \label{subsec:PDEs}

The 2D transient nonlinear heat equation is used to validate the effectiveness of the proposed method for tasks with different PDE parameters. The PDE's general form is defined in Eq. (\ref{Eq:2D_heat})

\begin{equation}
\begin{aligned}
& \frac{\partial u(x, y, t)}{\partial t} - \gamma \frac{\partial^2 u(x, y, t)}{\partial x^2} - \gamma \frac{\partial^2 u(x, y, t)}{\partial y^2} + k_1 k_2 \, {\rm tanh}[u(x, y, t)] = \mathcal{S}(x, y, t), \\ 
& \hspace{7cm} x \in [-1, 1], \; y \in [-1, 1], \; t \in [0, 1]
\label{Eq:2D_heat}
\end{aligned}
\end{equation}

\noindent where $\gamma$, $k_1$, and $k_2$ are PDE coefficients that can be varied to construct diverse PDE scenarios; $\mathcal{S}(x, y, t)$ is a problem-specific source term.


We distinguish between in-distribution and out-of-distribution scenarios based on the form of $\mathcal{S}(x, y, t)$. The in-distribution tasks are generated using a source term corresponding to the analytical solution in Eq. (\ref{Eq:2D_heat_id}). In contrast, the out-of-distribution tasks are constructed using source terms consistent with the analytical solutions specified in Eqs. (\ref{Eq:2D_heat_F1})-(\ref{Eq:2D_heat_F3}). 

We identify an optimal data-driven U-Net as a baseline (referenced as NAS-U-Net). Experimental details for NOVA and NAS-U-Net are presented in Section \ref{subsec:training details}. The comparison results of NOVA and NAS-U-Net on $\mathbfcal{T}_{\rm test\_id}$ and $\mathbfcal{T}_{\rm test\_ood}$ are collected in Fig. \ref{heat and flow mixing}e and Supplementary Table 5. 

\textbf{(1) In-distribution results}

A representative case in $\mathbfcal{T}_{\rm test\_id}$ is plotted in Fig. \ref{heat and flow mixing}a. Visually, NAS-U-Net and NOVA can both produce reasonable predictions, although NOVA is noticeably more accurate at later simulation times (similar to 2D diffusion-reaction). NOVA can successfully generalize across different PDE coefficients and has approximately two orders of magnitude lower RMSE than NAS-U-Net ($10^{-3}$ vs $10^{-1}$). 

\begin{figure}[htp]
\centering
\vspace*{0mm}
\includegraphics[width=1\textwidth]{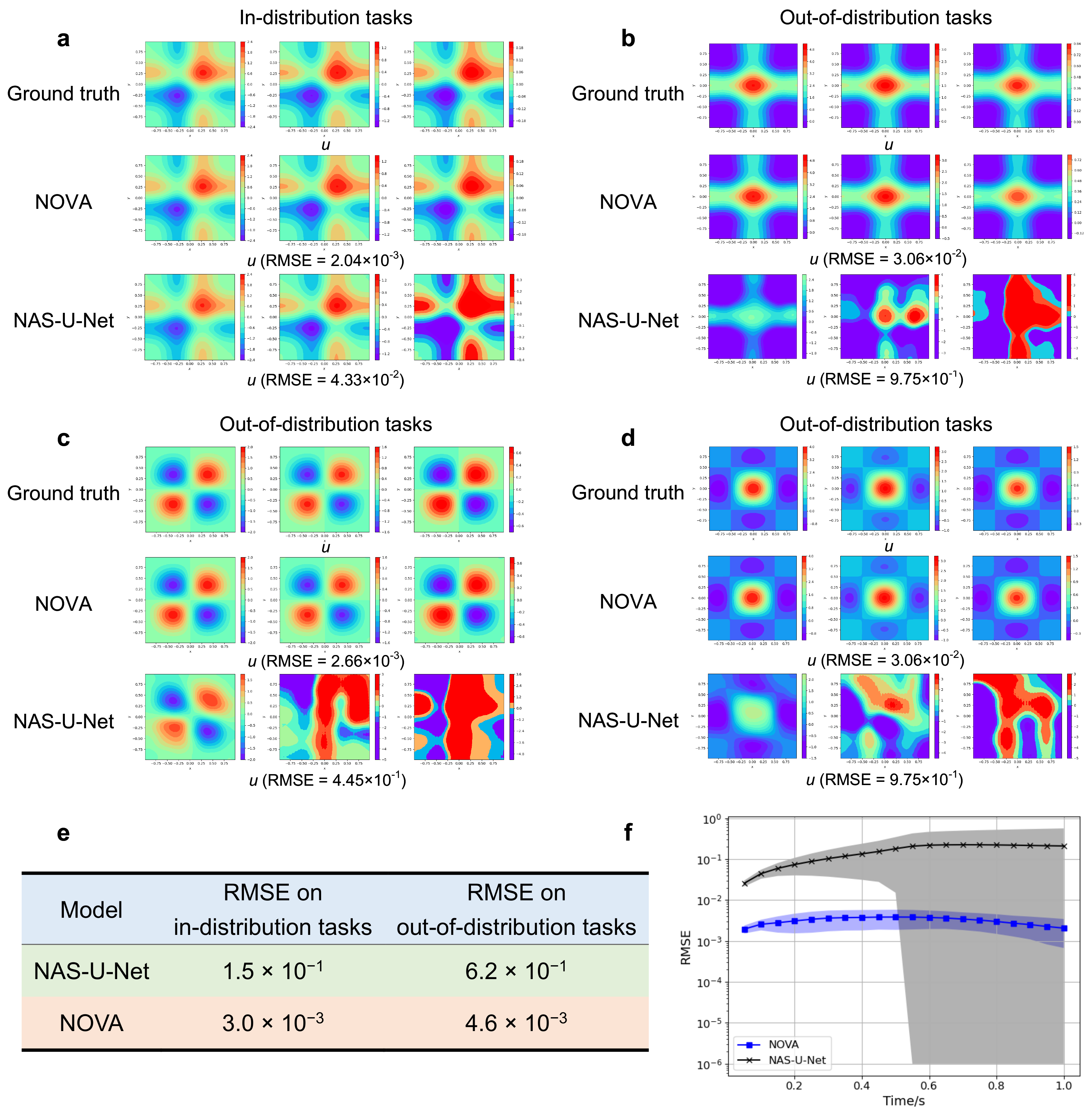}
\vspace*{-0mm}
\caption{Generalization performance of NOVA for the 2D nonlinear heat equation. \textbf{a-d}. Predictions from NOVA and NAS-U-Net on a representative case from (a) $\mathbfcal{T}_{\rm test\_id}$, (b) $\mathbfcal{T}_{{\rm test\_ood}\_F1}$, (c) $\mathbfcal{T}_{{\rm test\_ood}\_F2}$, and (d) $\mathbfcal{T}_{{\rm test\_ood}\_F3}$. Each subfigure displays predictions at $t = 0.05$, $t = 0.5$, and $t = 1$. RMSE in brackets is the average over the entire simulation time. \textbf{e}. Mean RMSE from NOVA and NAS-U-Net on $\mathbfcal{T}_{\rm test\_id}$ and $\mathbfcal{T}_{{\rm test\_ood}\_F3}$. \textbf{f}. Mean RMSE at each time step for tasks in $\mathbfcal{T}_{\rm test\_id}$ for NOVA and NAS-U-Net. Shaded area indicates the standard deviation.}
\vspace*{0mm}
\label{heat and flow mixing}
\end{figure}


The RMSE values at each time step for NOVA and NAS-U-Net are depicted in Fig. \ref{heat and flow mixing}f. The RMSE values of NOVA on $\mathbfcal{T}_{\rm test\_id}$ at each time step remain relatively low compared with NAS-U-Net. This lower error illustrates NOVA's ability to predict more accurately over a longer time horizon due to the lower rate of error accumulation in the autoregressive rollout. 

\textbf{(2) Out-of-distribution results}

To further illustrate the performance of NOVA for scenarios which are dissimilar to the original training examples, we test NOVA and NAS-U-Net on $\mathbfcal{T}_{\rm test\_ood}$. Sample prediction results for $\mathbfcal{T}_{{\rm test\_ood}\_F1}$, $\mathbfcal{T}_{{\rm test\_ood}\_F2}$, and $\mathbfcal{T}_{{\rm test\_ood}\_F3}$ are provided in Fig. \ref{heat and flow mixing}b-Fig. \ref{heat and flow mixing}d. 

Despite the distributional shift, NOVA can still provide accurate predictions due to its physics-guided inductive bias and physics-informed adaptation. In contrast, the NAS-U-Net performs poorly on these new PDE scenarios relative to its performance on the original training dataset. As a result, the RMSE of NOVA is more than one order of magnitude lower than NAS-U-Net on $\mathbfcal{T}_{{\rm test\_ood}\_F1}$, and more than two orders of magnitude lower than NAS-U-Net on $\mathbfcal{T}_{{\rm test\_ood}\_F2}$ and $\mathbfcal{T}_{{\rm test\_ood}\_F3}$. NOVA also exhibits more than one order of magnitude better performance than the pure data-driven model (NAS-U-Net) on other metrics (i.e., nRMSE, max error, cRMSE, and bRMSE) as summarized in Fig. \ref{heat and flow mixing}e and Supplementary Table 5. 

\textbf{Extension of NOVA to extreme case with no data but known parametric range of interest.} We apply NOVA to a separate flow-mixing problem, and demonstrate that NOVA can also be adapted to produce good predictions for different PDE parameters in the extreme case where there is no data. These experiments are in Supplementary Section 5.

\section{Discussion}\label{sec:Discussion}

This work presents NOVA, an innovative paradigm that integrates neural physics solver discovery with a physics-informed adaptation step to automatically identify models with inductive biases appropriate for physics prediction. By embedding physical laws during prediction, NOVA enables rapid and versatile, data-free adaptation for robust and accurate predictions even when the new system may be different from the training distribution. In contrast to existing data-driven or physics-informed approaches, NOVA incorporates physical consistency and robustness to out-of-distribution scenarios as an explicit part of the search process. 

Experimental results across diverse PDE problems including nonlinear heat, flow mixing, diffusion-reaction, and Navier-Stokes equations, confirm that NOVA consistently learns feature representations aligned to the physics of interest, conferring robustness to distributional shifts in ICs, geometries, or the PDE itself. Its overall performance across physical systems is summarized in Supplementary Table 1. Notably, the relative performance of NOVA to baseline models (``Gain'') remains high across all out-of-distribution test scenarios (an average of $99 \times$), and, more remarkably, is increased relative to the ``Gain'' for in-distribution test scenarios.  

Through incorporation of physics, NOVA enables neural physics solvers to generalize under distributional shifts, with two key implications for their adoption in AI for Science: 1) NOVA-derived models have enhanced robustness to covariate shift for more accurate long-time horizon autoregressive rollout, a key benefit for the emulation of complex dynamical systems; 2) NOVA-derived models can flexibly serve as a data-light, yet fast prediction model for scientific discovery, where one may not be able to \textit{a priori} tightly define the parametric bounds of interest for the physical system, such as in performance-guided generative diffusion. 

Beyond the demonstrated performance benefits, the results also clearly show that models with higher training errors may still generalize better if the architecture possesses the appropriate inductive bias, highlighting the connection between any structural prior as dictated by the architecture and specific physics. Further investigation of the alignment between architecture and physics may be worth elucidating in future work, especially as we believe the structural prior (as represented by the neural architecture) may be a complementary source of physics-regularization as one may regard traditional physics-informed losses (e.g., in PINNs).
A better understanding of this relationship may also facilitate future extensions of this work to more complex settings such as multi-physics or three-dimensional problems where data is likely to be even more expensive to obtain and venturing out-of-distribution is consequently even more likely.

Overall, NOVA represents a principled approach for enhancing model robustness over and beyond their training data distributions by embedding appropriate physical structure in the neural architecture and physics-compliance into the adaptation step. This paradigm of co-designing neural architectures with physics-informed adaptation presents a promising path towards more dependable and robust scientific machine learning and neural physics solvers. 



\section{Methods}\label{sec:Methods}

\subsection{NOVA framework}\label{subsec:framework}


In the NOVA framework, the model is trained across different PDE scenarios $\mathbfcal{T}$ with the data loss $\mathcal{L}_{\rm Data}$ as per Eq. (\ref{Eq:data loss})

\begin{equation}
\begin{aligned}
    & \mathcal{L}_{\rm Data} = \frac{1}{n_{\rm Data}} \sum_{i = 1}^{n_{\rm Data}} \big\lVert \boldsymbol{u}(\boldsymbol{x}_i, t_i) - \boldsymbol{u}_{gt}(\boldsymbol{x}_i, t_i) \big\rVert_2^{2}
\label{Eq:data loss}
\end{aligned}
\end{equation}

\noindent where $\boldsymbol{u}_{gt}(\boldsymbol{x}, t)$ is the ground truth (labeled data) for the PDE domain as encompassed by $\boldsymbol{\Omega} \times [0, t_m]$. 

The trained model can then be used to predict the output $\boldsymbol{u}$ for any new scenario with a physics-informed adaptation step that ensures $\boldsymbol{u}$ adheres to the specific governing physics (i.e., PDEs, BCs, and ICs) described in Eq. (\ref{Eq:PINN}). 

To evaluate the generalization performance of the architecture across a family of tasks $\mathbfcal{T}$ in our available data, we further divide $\mathbfcal{T}$ into a training dataset $\mathbfcal{T}_{\rm training}$ (including NAS training tasks $\mathbfcal{T}_{\rm training\_NAS}$ and NAS validation tasks $\mathbfcal{T}_{\rm validation\_NAS}$). 

Building on this, NOVA's progressive learning strategy can be seen to have the following two key components. 

\textbf{(1) Lightweight training for optimized kernel parameters $\boldsymbol{\omega}^*$}

Without loss of generality, a convolutional neural network (CNN) type architecture is chosen as the backbone architecture for NOVA in this work, and for description in the rest of this section. By incorporating an additional convolution layer as the last layer prior to the final feature map for regression \cite{9451544}, the outputs of the CNN with spatial dimensions $x$, $y$ and time dimension $t$ can be expressed by 

\begin{small}
\begin{equation}
\begin{aligned}
    & u_j(x, y, t; \theta_{ij}) = \sum_{i = 1}^{c_{\rm in}} \sum_{n_x = 1}^{n_k} \sum_{n_y = 1}^{n_k} f_i(x + (n_x - 1)\Delta x, y + (n_y - 1)\Delta y, t; \boldsymbol{\omega}) \cdot \theta_{ij}(n_x, n_y), \\
    & j = 1, 2,  \ldots, c_{\rm out}, \; x \in \{x_0, x_1, \cdots, x_m\}, \; y \in \{y_0, y_1, \cdots, y_m\}, \; t \in \{t_0, t_1, \cdots, t_m\}
\label{Eq:2D_CNN}
\end{aligned}
\end{equation}
\end{small}

\noindent where $u_j(x, y, t; \theta_{ij})$ is the final feature map value of the $j$-th channel at $(x, y, t)$; $\Delta x$ and $\Delta y$ are the discretization distances along $x$ and $y$ axes, respectively; $f_i(x + (n_x - 1)\Delta x, y + (n_y - 1)\Delta y, t; \boldsymbol{\omega})$ is the feature map of the $i$-th channel at $(x + (n_x - 1)\Delta x, y + (n_y - 1)\Delta y, t)$ of the penultimate layer; $\boldsymbol{\omega}$ are the kernel parameters related to $f_i$; $\theta_{ij}(n_x, n_y)$ is the final layer parameter value at $(n_x, n_y)$; $c_{\rm in}$ is the number of channels of the input into the last layer; $c_{\rm out}$ is the number of channel(s) of the output after the last layer; $n_k$ is the kernel size.

Eq. (\ref{Eq:2D_CNN}) indicates that the feature maps of the penultimate layer (i.e., $\boldsymbol{f}$) are derived from the preceding kernel parameters $\boldsymbol{\omega}$, while the CNN outputs are generated by applying the $n_k \times n_k$ kernels $\boldsymbol{\theta}$ on $\boldsymbol{f}$. The process is graphically expressed in Fig. \ref{Fig:last_layer}.

\begin{figure}[htp]
\centering
\vspace*{0mm}
\includegraphics[width=0.8\textwidth]{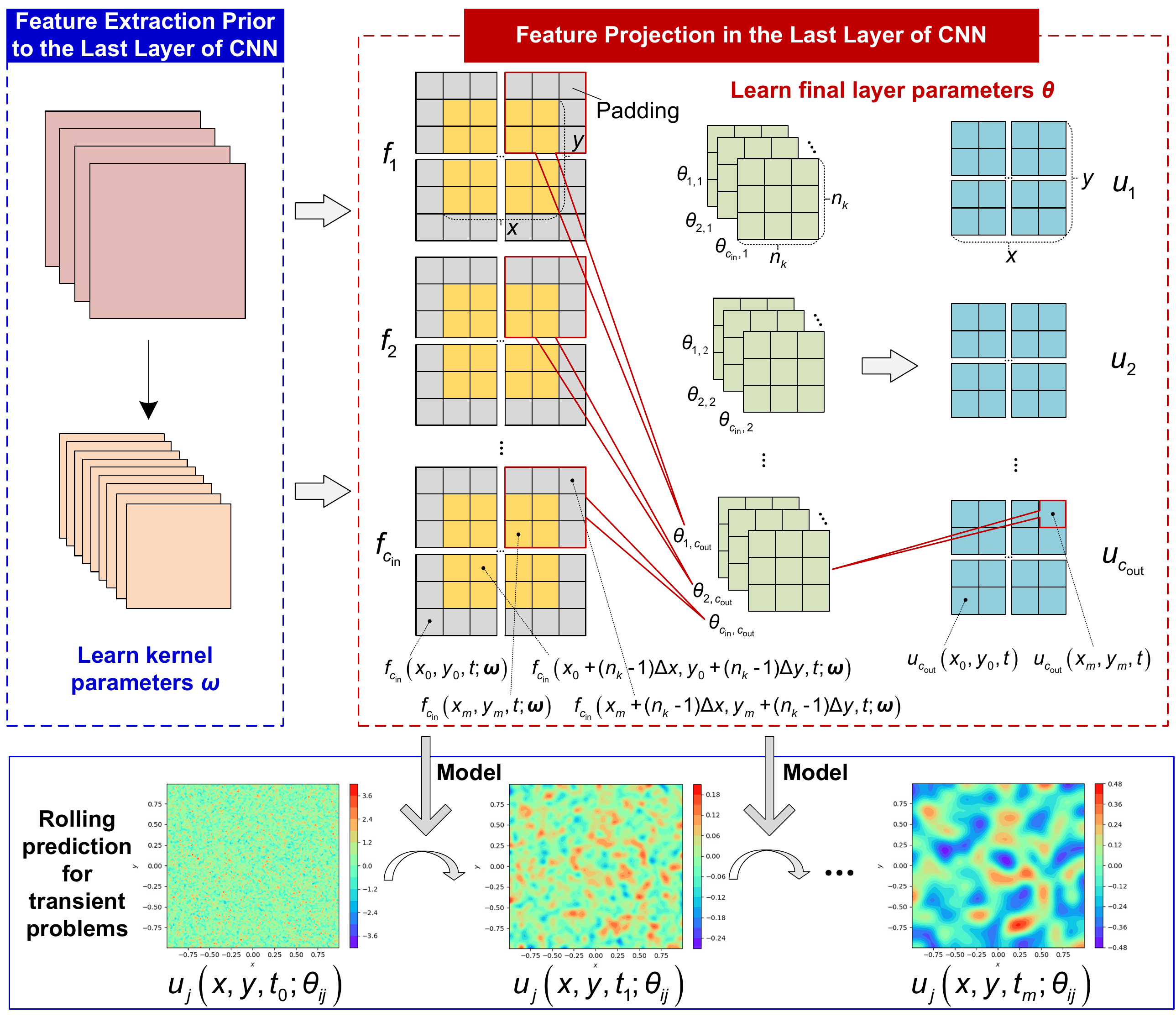}
\vspace*{-0mm}
\caption{CNN architecture diagram. The kernel parameters before the CNN last layer $\boldsymbol{\omega}$ are learned across multiple tasks for more generalizable feature representations. The final layer parameters $\boldsymbol{\theta}$ are learned to project the extracted features for each new PDE scenario. In this work, an autoregressive rollout is employed to handle transient problems.}
\vspace*{0mm}
\label{Fig:last_layer}
\end{figure}

The optimized kernel parameters $\boldsymbol{\omega}^*$ prior to the last layer of the CNN are learned through several epochs of training with the data loss $\mathcal{L}_{\rm Data}$, as computed using $\mathbfcal{T}_{\rm training}$ in Eq. (\ref{Eq:training_kernel_formulation}) 

\begin{equation}
\boldsymbol{\omega}^* = \boldsymbol{\omega}^{(0)}-\sum_{i = 0}^{N-1} \eta \nabla_{\boldsymbol{\omega}}\mathcal L_{\rm Data}(\boldsymbol{\omega}^{(i)}, \mathbfcal{T}_{\rm training})
\label{Eq:training_kernel_formulation}
\end{equation}

\textbf{(2) Physics-informed zero-shot adaptation via optimizing final layer parameters $\boldsymbol{\theta}^*$}

For each subsequent task, physics-informed zero-shot adaptation is applied to the CNN's last layer to get optimized final layer parameters $\boldsymbol{\theta}^*$ by enforcing the relevant physics. The solution for $\boldsymbol{\theta}^*$ can be obtained by incorporating known physics (i.e., PDEs, BCs, and ICs from Eq. (\ref{Eq:PINN})) with Eq. (\ref{Eq:2D_CNN}) and Eq. (\ref{Eq:training_kernel_formulation}) and formulating the problem as a convex optimization problem in Eq. (\ref{Eq:training_formulation}) which can be efficiently and accurately solved via Tikhonov regularization (also called ridge regression) 





\begin{equation}
\boldsymbol{\theta}^* = \mathop{\arg\min}\limits_{\boldsymbol{\theta}} \ [(\boldsymbol{A}_{\boldsymbol{\omega}^*}^{\mathbfcal{T}_{\rm validation}} \boldsymbol{\theta} - \boldsymbol{b}^{\mathbfcal{T}_{\rm validation}})^\top (\boldsymbol{A}_{\boldsymbol{\omega}^*}^{\mathbfcal{T}_{\rm validation}} \boldsymbol{\theta} - \boldsymbol{b}^{\mathbfcal{T}_{\rm validation}}) + \lambda \boldsymbol{\theta}^\top \boldsymbol{\theta}]
\label{Eq:training_formulation}
\end{equation}


\noindent where $\boldsymbol{A}_{\boldsymbol{\omega}^*}^{\mathbfcal{T}_{\rm validation}}$ is the physics-informed matrix related to the outputs obtained through the use of kernel parameters before the CNN's last layer $\boldsymbol{\omega}^*$ for $\mathbfcal{T}_{\rm validation}$; $\boldsymbol{\theta}$ are the CNN final layer parameters; $\boldsymbol{A}_{\boldsymbol{\omega}^*}^{\mathbfcal{T}_{\rm validation}} \boldsymbol{\theta}$ and $\boldsymbol{b}^{\mathbfcal{T}_{\rm validation}}$ are the left and right sides of the physics-informed least-squares problem, respectively; $\lambda$ is the Tikhonov regularization parameter to better condition the convex least-squares solution; $\eta$ is the learning rate. The process to obtain $\boldsymbol{\theta}^*$ for the CNN's last layer is further detailed in Section \ref{subsec:Pseudo-inverse learning}. 

\textbf{Theoretical Analysis.} A key insight is that the model's representational capacity improves sublinearly with additional training epochs and that most of the performance improvement occurs early, especially given that the last layer is to be solved by a convex Tikhonov regularization step. Hence, stopping the learning of $\boldsymbol{\omega}^*$ early and using the physics-informed adaptation step to achieve the best possible prediction can yield strong generalization performance with minimal training as necessary for NOVA. A more complete theoretical analysis is detailed in Supplementary Section 9.

\subsection{Physics-guided neural architecture search for NOVA}\label{subsec:nas_framework}
\subsubsection{Formulation of NAS problem} 


A novel physics-guided NAS framework is proposed to autonomously uncover better architectures with the appropriate inductive biases for physics problems. This involves a reformulation of the entire NAS training and validation methodology. In view of previously demonstrated powerful and robust performance across various fields such as fluid dynamics and generative AI, a generalized U-Net is employed as the architecture backbone (described in Section \ref{NAS search space}).

The physics-guided NAS consists of a bi-level optimization framework. The lower-level optimization aims to refine the kernel parameters and final layer parameters for a given architecture. Meanwhile, the upper-level optimization searches for the best architecture and provides inputs to guide the lower-level optimization. The mathematical formulation of the NAS problem is as per Eq. (\ref{Eq:NAS problem})

\begin{equation}
\label{Eq:NAS problem}
\begin{aligned}
    & {\rm find} \quad \boldsymbol{X} = [\boldsymbol{x}_{\rm encoding }, \, \lambda, \, \lambda_{\rm BC}, \, \lambda_{\rm IC}] \\
    & {\rm min} \quad f_p(\boldsymbol{X}; \ \boldsymbol{\omega}^*, \, \boldsymbol{\theta}^*) \\
    & {\rm s.t.} \quad \ \boldsymbol{\omega}^*, \, \boldsymbol{\theta}^* \Leftarrow{} \ \text{Eqs. (7) and (8)}
\end{aligned}
\end{equation}

\noindent where $\boldsymbol{X}$ are the design variables of the NAS problem; $\boldsymbol{x}_{\rm encoding}$ represents the encoding sequence for the architecture representation; $f_p$ is the predictive performance on the validation dataset, evaluated using RMSE; $\lambda$ is the regularization parameter in Eq. (\ref{Eq:training_formulation}); $\lambda_{\rm BC}$ and $\lambda_{\rm IC}$ are the scaling factors for enforcing BCs and ICs in Eq. (\ref{Eq:pseudo_inverse}). 

The NAS and evaluation workflow of NOVA is detailed in Section \ref{subsec:Process}, and is briefly summarized as follows. 

\textbf{Step 1:} For a given architecture $\boldsymbol{x}_{\rm encoding}$, a few epochs are performed on the NAS training dataset $\mathbfcal{T}_{\rm training\_NAS}$ to optimize the kernel parameters $\boldsymbol{\omega}^*$. If training data is available, the data loss defined in Eq. (\ref{Eq:data loss}) is used; otherwise, the physics loss defined in Supplementary Eq. (2) can be employed.

\textbf{Step 2:} A convexified physics-informed adaptation is applied to obtain optimized final layer parameters $\boldsymbol{\theta}^*$ for each task in the NAS validation dataset $\mathbfcal{T}_{\rm validation\_NAS}$. The predictive performance (i.e., RMSE in this work) of the architecture is then evaluated on $\mathbfcal{T}_{\rm validation\_NAS}$.

\textbf{Step 3:} Based on the current architectures and their associated performance, the optimizer generates a set of promising candidate architectures and corresponding scaling factors (i.e., design variables $\boldsymbol{X}$).

\textbf{Step 4:} If the termination condition is not met, the process returns to Step 1. Otherwise, the best-found architecture is selected for retraining and final evaluation on the test tasks $\mathbfcal{T}_{\rm test}$, which include both in-distribution tasks $\mathbfcal{T}_{\rm test\_id}$ and out-of-distribution tasks $\mathbfcal{T}_{\rm test\_ood}$.

In this work, we adapt a previously developed adaptive surrogate-assisted evolutionary algorithm \cite{wei2025continuous} as the optimizer to efficiently solve the NAS problem. Briefly, a Kriging model is employed to approximate the predictive performance of each candidate neural architecture. To effectively balance exploration and exploitation, we employ a two-level search strategy. The evolution-based global exploration mechanism promotes architectural diversity and mitigates premature convergence, while a clustering-guided local infill sampling strategy refines the search around promising regions. This coordinated approach accelerates convergence and facilitates the discovery of high-performing architectures with significantly reduced computational cost. Further details of the NAS algorithm can be found in \cite{wei2025continuous}.

\subsubsection{NAS search space} \label{NAS search space}
In this work, a generalized U-Net is established as the backbone for NAS (illustrated in Supplementary Fig. 3). U-Net is an advanced CNN model with a characteristic U-shape architecture that has shown great success in various domains such as fluid dynamic regression, biomedical image segmentation, and generative AI \cite{ronneberger2015u,williams2024unified}. Similar to the basic U-Net, the generalized U-Net retains the conventional encoder-decoder architecture with skip connections. The main difference is that the operators and executive sequences in each downsampling and upsampling cells can be changed, thereby unveiling high-performing architectures through NAS. A continuous encoding method is established to describe the connections and executive sequences of the operators as per Eq. (\ref{Eq:encoding})

\begin{equation}
\begin{aligned}
    \boldsymbol{x}_{c\_i} = [n_{no1\_i}, \ n_{no2\_i}], \quad i = 3, 4, \cdots, (n_n + 2)
\label{Eq:encoding}
\end{aligned}
\end{equation}

\noindent where $\boldsymbol{x}_{c\_i}$ is the encoding sequence of each node, and each node receives two outputs from the previous cells or nodes \cite{liu2018darts}; $n_{no1\_i}$ and $n_{no2\_i}$ are the elements linked to the first and second cell or node for Node $i$, respectively; $n_n$ is the number of intermediate nodes of each cell. 

The integer parts of $n_{no1\_i}$ and $n_{no2\_i}$ are the indices of the first and the second cell or node linked to Node $i$, respectively. The fractional parts of $n_{no1\_i}$ and $n_{no2\_i}$ indicate the first and the second downsampling, upsampling or normal operators for Node $i$. The $j$-th operator of the downsampling, upsampling or normal operators is selected if the fractional part of $n_{no1\_i}$ or $n_{no2\_i}$ is located in $[(j-1)/n_o, j/n_o]$, where $n_o$ is the number of downsampling, upsampling or normal operators. The sets of downsampling, upsampling, and normal operators are selected based on their prevalence in the CNN literature, and are listed in Supplementary Table 7 \cite{9201169}. 



\subsubsection{Physics-informed zero-shot adaptation}\label{subsec:Pseudo-inverse learning}


The physics-informed zero-shot adaptation aims to find final layer parameters $\boldsymbol{\theta}^*$ in Eq. (\ref{Eq:2D_CNN}) such that the solution $u_j(x, y, t; \theta_{ij})$ best satisfies Eq. (\ref{Eq:PINN}). Utilizing Eq. (\ref{Eq:2D_CNN}), the physics-informed adaptation problem can be formulated as a convex problem through the substitution of a set of PDE, BC and/or IC collocation points for any new task(s) $\mathbfcal{T_{\rm new}}$ in Eq. (\ref{Eq:pseudo_inverse})

\begin{scriptsize}
\begin{equation}
\begin{array}{c}
\boldsymbol{A}_{\boldsymbol{\omega}^*}^{\mathbfcal{T_{\rm new}}} \boldsymbol{\theta} = \boldsymbol{b}^{\mathbfcal{T_{\rm new}}} \\ [15pt]
\begin{bmatrix}
    \cdots & \mathcal{N}[f_i^\mathbfcal{T_{\rm new}}(x_0^{\rm PDE} + (n_x - 1) \Delta x, y_0^{\rm PDE} + (n_y - 1) \Delta y, t; \boldsymbol{\omega}^*)] & \cdots \\
    \vdots & \vdots & \vdots \\
    \cdots & \mathcal{N}[f_i^\mathbfcal{T_{\rm new}}(x_{n_{\rm PDE}-1}^{\rm PDE} + (n_x - 1) \Delta x, y_{n_{\rm PDE}-1}^{\rm PDE} + (n_y - 1) \Delta y, t; \boldsymbol{\omega}^*)] & \cdots \\
    \quad & \quad & \quad \\ \hdashline
    \quad & \quad & \quad \\
    \cdots & \lambda_{\rm BC} \cdot \mathcal{B}[f_i^\mathbfcal{T_{\rm new}}(x_0^{\rm BC} + (n_x - 1) \Delta x, y_0^{\rm BC} + (n_y - 1) \Delta y, t; \boldsymbol{\omega}^*)] & \cdots \\
    \vdots & \vdots & \vdots \\
    \cdots & \lambda_{\rm BC} \cdot \mathcal{B}[f_i^\mathbfcal{T_{\rm new}}(x_{n_{\rm BC}-1}^{\rm BC} + (n_x - 1) \Delta x, y_{n_{\rm BC}-1}^{\rm BC} + (n_y - 1) \Delta y, t; \boldsymbol{\omega}^*)] & \cdots \\
    \quad & \quad & \quad \\ \hdashline
    \quad & \quad & \quad \\
    \cdots & \lambda_{\rm IC} \cdot f_i^\mathbfcal{T_{\rm new}}(x_0^{\rm IC} + (n_x - 1) \Delta x, y_0^{\rm IC} + (n_y - 1) \Delta y, 0; \boldsymbol{\omega}^*) & \cdots \\
    \vdots & \vdots & \vdots \\
    \cdots & \lambda_{\rm IC} \cdot f_i^\mathbfcal{T_{\rm new}}(x_{n_{\rm IC}-1}^{\rm IC} + (n_x - 1) \Delta x, y_{n_{\rm IC}-1}^{\rm IC} + (n_y - 1) \Delta y, 0; \boldsymbol{\omega}^*) & \cdots
\end{bmatrix}
\begin{bmatrix}
    \vdots \\
    \theta_{ij}(n_x, n_y) \\
    \vdots
\end{bmatrix}
= \\ [85pt]
\begin{bmatrix}
    \cdots & \mathcal{S^\mathbfcal{T_{\rm new}}}(x_0^{\rm PDE} + (n_x - 1) \Delta x, y_0^{\rm PDE} + (n_y - 1) \Delta y, t) & \cdots \\
    \vdots & \vdots & \vdots \\ 
    \cdots & \mathcal{S^\mathbfcal{T_{\rm new}}}(x_{n_{\rm PDE}-1}^{\rm PDE} + (n_x - 1) \Delta x, y_{n_{\rm PDE}-1}^{\rm PDE} + (n_y - 1) \Delta y, t) & \cdots \\ 
    \quad & \quad & \quad \\ \hdashline
    \quad & \quad & \quad \\
    \cdots & \lambda_{\rm BC} \cdot g^\mathbfcal{T_{\rm new}}(x_0^{\rm BC} + (n_x - 1) \Delta x, y_0^{\rm BC} + (n_y - 1) \Delta y, t) & \cdots \\
    \vdots & \vdots & \vdots \\
    \cdots & \lambda_{\rm BC} \cdot g^\mathbfcal{T_{\rm new}}(x_{n_{\rm BC}-1}^{\rm BC} + (n_x - 1) \Delta x, y_{n_{\rm BC}-1}^{\rm BC} + (n_y - 1) \Delta y, t) & \cdots \\
    \quad & \quad & \quad \\ \hdashline
    \quad & \quad & \quad \\
    \cdots & \lambda_{\rm IC} \cdot u_0^\mathbfcal{T_{\rm new}}(x_0^{\rm IC} + (n_x - 1) \Delta x, y_0^{\rm IC} + (n_y - 1) \Delta y) & \cdots \\
    \vdots & \vdots & \vdots \\
    \cdots & \lambda_{\rm IC} \cdot u_0^\mathbfcal{T_{\rm new}}(x_{n_{\rm IC}-1}^{\rm IC} + (n_x - 1) \Delta x, y_{n_{\rm IC}-1}^{\rm IC} + (n_y - 1) \Delta y) & \cdots
\end{bmatrix}
\label{Eq:pseudo_inverse}
\end{array}
\end{equation}
\end{scriptsize}

\noindent where $\lambda_{\rm BC}$ and $\lambda_{\rm IC}$ are the scaling factors to enforce BCs and ICs, respectively; $n_{\rm PDE}$, $n_{\rm BC}$, and $n_{\rm IC}$ are the respective number of PDE, BC, and IC points used to construct the matrix for determining the final layer parameters.

In Eq. (\ref{Eq:pseudo_inverse}), the spatial and temporal derivatives of the PDE part in $\boldsymbol{A}_{\boldsymbol{\omega}^*}^{\mathbfcal{T_{\rm new}}}$ are calculated through the method described in Supplementary Section 6. Once $\boldsymbol{A}_{\boldsymbol{\omega}^*}^{\mathbfcal{T_{\rm new}}}$ and $\boldsymbol{b}^{\mathbfcal{T_{\rm new}}}$ are defined, the optimal final layer parameters $\boldsymbol{\theta}^*$ can be calculated by a rapid Tikhonov regularization solution for Eq. (\ref{Eq:pseudo_inverse}) as per Eq. (\ref{Eq:theta}).

\begin{equation}
\boldsymbol{\theta}^*= \left\{
\begin{aligned}
    & \Big[\lambda \boldsymbol{I} + \big(\boldsymbol{A}_{\boldsymbol{\omega}^*}^{\mathbfcal{T_{\rm new}}}\big)^\mathrm{T} \boldsymbol{A}_{\boldsymbol{\omega}^*}^{\mathbfcal{T_{\rm new}}}\Big]^{-1} \big(\boldsymbol{A}_{\boldsymbol{\omega}^*}^{\mathbfcal{T_{\rm new}}}\big)^\mathrm{T} \boldsymbol{b}^{\mathbfcal{T_{\rm new}}}, & \text{if the system is over-determined} \\
    & \big(\boldsymbol{A}_{\boldsymbol{\omega}^*}^{\mathbfcal{T_{\rm new}}}\big)^\mathrm{T}
    \Big[\lambda \boldsymbol{I} +  \boldsymbol{A}_{\boldsymbol{\omega}^*}^{\mathbfcal{T_{\rm new}}} \big(\boldsymbol{A}_{\boldsymbol{\omega}^*}^{\mathbfcal{T_{\rm new}}}\big)^\mathrm{T} \Big]^{-1} \boldsymbol{b}^{\mathbfcal{T_{\rm new}}}, & \text{if the system is under-determined}
\end{aligned}
\right.
\label{Eq:theta}
\end{equation}

Note that $\boldsymbol{\theta}^*$ can be calculated directly by Eq. (\ref{Eq:theta}) in one step if the PDE system of Eq. (\ref{Eq:PINN}) is linear. For nonlinear PDEs, the lagging of coefficient method is used \cite{pletcher2012computational}. 

\subsubsection{NAS and prediction with NOVA}\label{subsec:Process}

The NAS and prediction workflow of NOVA is illustrated in Fig. \ref{Fig:whole_process}. The entire process comprises two phases, i.e., NAS and test (prediction for a new scenario). The NAS phase searches for optimal neural architecture that performs well on the validation dataset through a data-free physics-informed adaptation. The test phase uses a final selected architecture to compute predictions for test PDE scenarios of interest, including for instances which vary from the training data distribution.

\begin{figure}[htp]
\centering
\vspace*{0mm}
\includegraphics[width=1.0\textwidth]{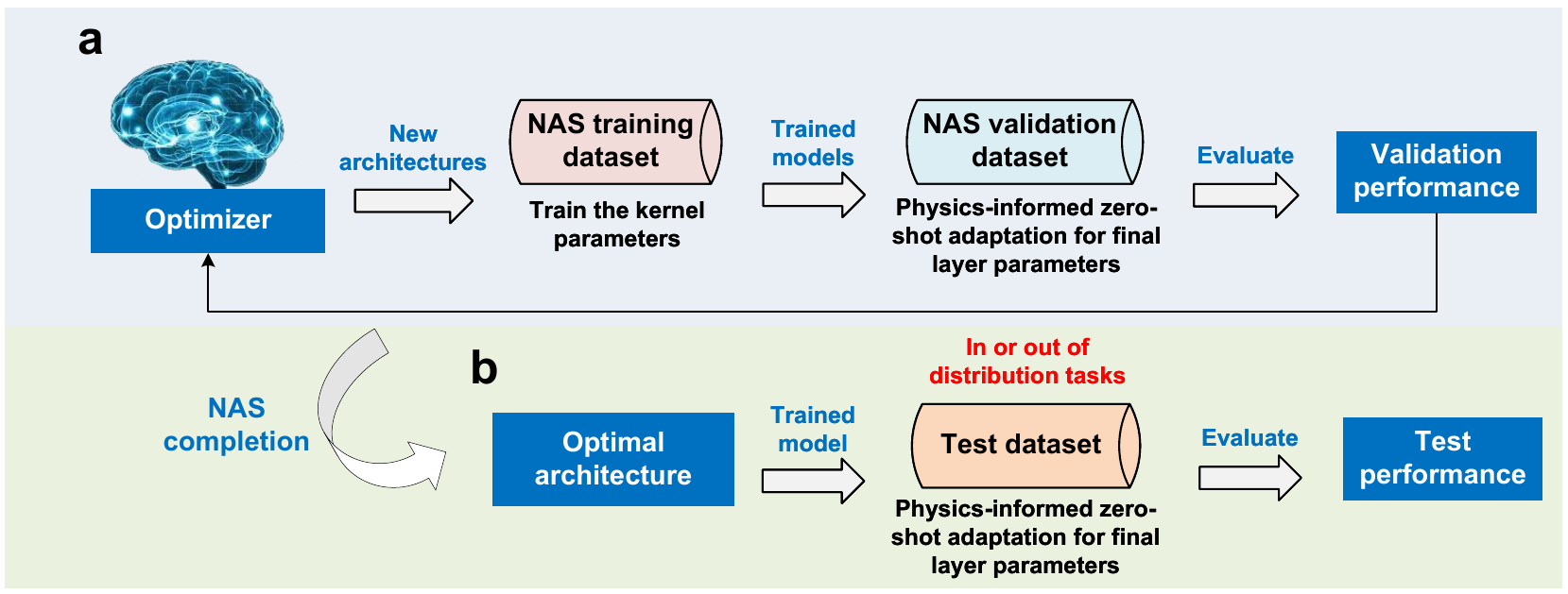}
\vspace*{-0mm}
\caption{NAS and prediction workflow of NOVA. \textbf{a}. NAS stage. In the NAS stage, the optimizer iteratively proposes promising architectures for evaluation. For each candidate architecture, lightweight training is conducted on the NAS training dataset $\mathbfcal{T}_{\rm training\_NAS}$ to get the optimized kernel parameters $\boldsymbol{\omega}^*$. Subsequently, physics-informed zero-shot adaptation is applied to learn the final layer parameters $\boldsymbol{\theta}^*$, after which the validation performance is assessed on $\mathbfcal{T}_{\rm validation\_NAS}$. This process continues until a predefined termination criterion is met. The architecture with the best validation performance is then selected for testing. \textbf{b}. Test stage. In the test stage, the selected optimal architecture undergoes lightweight retraining on the entire training dataset $\mathbfcal{T}_{\rm training}$ to obtain the refined kernel parameters $\boldsymbol{\omega}^*$. Physics-informed zero-shot adaptation is then applied to both in-distribution and out-of-distribution tasks in $\mathbfcal{T}_{\rm test}$ to learn the final layer parameters $\boldsymbol{\theta}^*$ and perform the final evaluation.}
\vspace*{0mm}
\label{Fig:whole_process}
\end{figure}

The physics-guided NAS process is outlined in Algorithm \ref{algorithm_NAS}. As depicted in Fig. \ref{Fig:whole_process}, the NAS process consists of two levels: the upper level searches for the best architecture and provides inputs for the lower level, while the lower level evaluates the validation performance of the input architecture through physics-informed zero-shot adaptation. 

Initially, as the current number of function evaluations $\rm NFE$ is less than the maximum number of function evaluations $\rm NFE_{max}$, the optimizer generates new encoding sequences $\boldsymbol{x}_{\rm encoding}$ (i.e., new architectures), corresponding regularization parameter $\lambda$, BC scaling factor $\lambda_{\rm BC}$, and IC scaling factor $\lambda_{\rm IC}$ for subsequent training and evaluation. During training, several epochs are applied to the new promising architectures on the NAS training dataset $\mathbfcal{T}_{\rm training\_NAS}$ with data loss as defined in Eq. (\ref{Eq:data loss}) to learn the features and obtain the optimized kernel parameters $\boldsymbol{\omega}^*$. 

For each task in the NAS validation dataset $\mathbfcal{T}_{\rm validation\_NAS}$, the final layer parameters at time step $t_p$, i.e., $\boldsymbol{\theta}^{*\mathcal{T}_j(t_p)}$, are calculated by Eq. (\ref{Eq:theta}) to guarantee the solution satisfies the underlying physics. For transient scenarios, predictions for each time-step is obtained via an autoregressive procedure. This process is repeated until solutions at all required time steps are obtained. The RMSE on $\mathbfcal{T}_{\rm validation\_NAS}$ for the new architectures, denoted as ${\rm RMSE}_{\rm validation}$, is calculated and saved into the database. Once $\rm NFE$ exceeds $\rm NFE_{max}$, the best architecture with the minimal ${\rm RMSE}_{\rm validation}$ is selected from the database.

The test process of NOVA, outlined in Algorithm \ref{algorithm_test}, is similar to the NAS validation process. The key difference is that only the best architecture obtained from Algorithm \ref{algorithm_NAS} is used. 

\begin{algorithm}
\caption{NAS for NOVA}\label{algorithm_NAS}
\begin{algorithmic}[1]
\renewcommand{\algorithmicrequire}{\textbf{Input:}}
\renewcommand{\algorithmicensure}{\textbf{Output:}}
\Require NAS training dataset $\mathbfcal{T}_{\rm training\_NAS}$, NAS validation dataset $\mathbfcal{T}_{\rm validation\_NAS}$, maximum number of function evaluations of optimizer ${\rm NFE_{max}}$, initial condition $\boldsymbol{u}(x, y, 0)$
\Ensure Best architecture $\boldsymbol{x}_{\rm encoding}^*$, best regularization parameter $\lambda^*$, best BC scaling factor $\lambda^*_{\rm BC}$, best IC scaling factor $\lambda^*_{\rm IC}$
\While{${\rm NFE} < {\rm NFE_{max}}$}
    \State $\boldsymbol{x}_{\rm encoding}, \lambda, \lambda_{\rm BC}, \lambda_{\rm IC} \Leftarrow {\rm Optimizer}$
    \For{$\mathcal{T}_i \in \mathbfcal{T}_{\rm training\_NAS}$}
        \State $\boldsymbol{\omega}^* \Leftarrow {\rm GradientDescent}(\boldsymbol{x}_{\rm encoding}, \mathcal{T}_i)$
    \EndFor
    \For{$\mathcal{T}_j \in \mathbfcal{T}_{\rm validation\_NAS}$}
        \For{$p = 1, 2, \cdots, n_t$}
            \State $\boldsymbol{\theta}^{*\mathcal{T}_j(t_p)} \Leftarrow {\rm PseudoInverse}(\boldsymbol{x}_{\rm encoding}, \boldsymbol{\omega}^*, \lambda, \lambda_{\rm BC}, \lambda_{\rm IC}, \mathcal{T}_j,$
            \State \qquad \qquad \quad $\boldsymbol{u}^{\mathcal{T}_j}(x, y, (p-1)\Delta t))$
            \For{$k = 1, 2, \cdots, N$}
                \State $\boldsymbol{\theta}^{*\mathcal{T}_j(t_p)} \Leftarrow {\rm LaggingCoefficientPseudoInverse}(\boldsymbol{x}_{\rm encoding}, \boldsymbol{\omega}^*, \lambda, \lambda_{\rm BC},$
                \State \qquad \qquad \quad $\lambda_{\rm IC}, \boldsymbol{\theta}^{*\mathcal{T}_j(t_p)}, \mathcal{T}_j, \boldsymbol{u}^{\mathcal{T}_j}(x, y, (p-1)\Delta t)$
            \EndFor
            \State $\boldsymbol{u}^{\mathcal{T}_j}(x, y, p\Delta t) \Leftarrow {\rm Prediction}(\boldsymbol{x}_{\rm encoding}, \boldsymbol{\omega}^*, \lambda, \lambda_{\rm BC}, \lambda_{\rm IC}, \boldsymbol{\theta}^{*\mathcal{T}_j(t_p)}, \mathcal{T}_j,$
            \State \qquad \qquad \qquad \qquad $\boldsymbol{u}^{\mathcal{T}_j}(x, y, (p-1)\Delta t))$
        \EndFor
        \State ${{\rm RMSE}_{\rm validation}^{\mathcal{T}_j}} \Leftarrow {\rm Evaluation}(\boldsymbol{u}^{\mathcal{T}_j}(x, y, p\Delta t), \mathcal{T}_j)$
    \EndFor
    \State ${\rm RMSE}_{\rm validation} \Leftarrow {\rm Mean}({{\rm RMSE}_{\rm validation}^{\mathcal{T}_j}})$
    \State Save $\boldsymbol{x}_{\rm encoding}, \lambda, \lambda_{\rm BC}, \lambda_{\rm IC}, \boldsymbol{u}^{\mathcal{T}_j}(x, y, p\Delta t), {\rm RMSE}_{\rm validation}^{\mathcal{T}_j}, {\rm RMSE}_{\rm validation}$ into Database
    \State ${\rm NFE = NFE} + 1$
\EndWhile
\State $\boldsymbol{x}_{\rm encoding}^*, \lambda^*, \lambda^*_{\rm BC}, \lambda^*_{\rm IC} \Leftarrow {\rm Selection(Database)}$
\end{algorithmic}
\end{algorithm}

\begin{algorithm}
\caption{Test for NOVA}\label{algorithm_test}
\begin{algorithmic}[1]
\renewcommand{\algorithmicrequire}{\textbf{Input:}}
\renewcommand{\algorithmicensure}{\textbf{Output:}}
\Require NAS training dataset $\mathbfcal{T}_{\rm training\_NAS}$, NAS validation dataset $\mathbfcal{T}_{\rm validation\_NAS}$, test dataset $\mathbfcal{T}_{\rm test}$, initial condition $\boldsymbol{u}(x, y, 0)$, best architecture $\boldsymbol{x}_{\rm encoding}^*$, best regularization parameter $\lambda^*$, best BC scaling factor $\lambda^*_{\rm BC}$, best IC scaling factor $\lambda^*_{\rm IC}$ 
\Ensure Test root mean squared error ${{\rm RMSE}_{\rm test}^{\mathcal{T}_j}}$, best kernel parameters before CNN last layer $\boldsymbol{\omega}^*$
\For{$\mathcal{T}_i \in \mathbfcal{T}_{\rm training\_NAS} \cup \mathbfcal{T}_{\rm validation\_NAS}$}
    \State $\boldsymbol{\omega}^* \Leftarrow {\rm GradientDescent}(\boldsymbol{x}_{\rm encoding}^*, \mathcal{T}_i)$
\EndFor
\For{$\mathcal{T}_j \in \mathbfcal{T}_{\rm test}$}
    \For{$p = 1, 2, \cdots, n_t$}
        \State $\boldsymbol{\theta}^{\mathcal{T}_j(t_p)} \Leftarrow {\rm PseudoInverse}(\boldsymbol{x}_{\rm encoding}^*, \boldsymbol{\omega}^*, \lambda^*, \lambda^*_{\rm BC}, \lambda^*_{\rm IC}, \mathcal{T}_j,$
        \State \qquad \qquad \ \, $\boldsymbol{u}^{\mathcal{T}_j}(x, y, (p-1)\Delta t))$
        \For{$k = 1, 2, \cdots, N$}
            \State $\boldsymbol{\theta}^{\mathcal{T}_j(t_p)} \Leftarrow {\rm LaggingCoefficientPseudoInverse}(\boldsymbol{x}_{\rm encoding}^*, \boldsymbol{\omega}^*, \lambda^*, \lambda^*_{\rm BC}, \lambda^*_{\rm IC},$
            \State \qquad \qquad \ \, $\boldsymbol{\theta}^{\mathcal{T}_j(t_p)}, \mathcal{T}_j, \boldsymbol{u}^{\mathcal{T}_j}(x, y, (p-1)\Delta t)$
        \EndFor
        \State $\boldsymbol{u}^{\mathcal{T}_j}(x, y, p\Delta t) \Leftarrow {\rm Prediction}(\boldsymbol{x}_{\rm encoding}^*, \boldsymbol{\omega}^*, \lambda^*, \lambda^*_{\rm BC}, \lambda^*_{\rm IC}, \boldsymbol{\theta}^{\mathcal{T}_j(t_p)}, \mathcal{T}_j,$
        \State \qquad \qquad \qquad \qquad $\boldsymbol{u}^{\mathcal{T}_j}(x, y, (p-1)\Delta t))$
    \EndFor
    \State ${{\rm RMSE}_{\rm test}^{\mathcal{T}_j}} \Leftarrow {\rm Evaluation}(\boldsymbol{u}^{\mathcal{T}_j}(x, y, p\Delta t), \mathcal{T}_j)$
\EndFor
\end{algorithmic}
\end{algorithm}

\subsection{Data generation}\label{subsec:data generation}

\subsubsection{Diffusion-reaction equations}\label{subsubsec:DR}


For the in-distribution tasks, the dataset is directly adopted from \cite{takamoto2022pdebench}. In \cite{takamoto2022pdebench}, 1000 cases with different ICs are generated, comprising 900 for training ($\mathbfcal{T}_{\rm training}$) and 100 for test ($\mathbfcal{T}_{\rm test\_id}$). In this work, the 900 training cases $\mathbfcal{T}_{\rm training}$ are split into 810 cases for NAS training $\mathbfcal{T}_{\rm training\_NAS}$ and 90 cases for NAS validation $\mathbfcal{T}_{\rm validation\_NAS}$. The ground truth for $\mathbfcal{T}_{\rm training}$ and $\mathbfcal{T}_{\rm test\_id}$ is directly obtained from \cite{takamoto2022pdebench}. 

For the out-of-distribution data, we generate distinct ICs that differ significantly from those in \cite{takamoto2022pdebench}. The ground truth for $\mathbfcal{T}_{\rm test\_ood}$ is generated using an in-house second-order finite-volume solver. Simulations are performed over the domain $x, y \in [-1, 1]$ from $t = 0$ to $t = 5$. To ensure temporal accuracy, a third-order strong stability-preserving Runge-Kutta method \cite{Gottlieb01SSPRK} is employed with a time step size of $1 \times 10^{-3}$. Zero flux Neumann BCs are applied.

The datasets $\mathbfcal{T}_{\rm training}$, $\mathbfcal{T}_{\rm test\_id}$, and $\mathbfcal{T}_{\rm test\_ood}$ are each discretized into $n_x = 128$, $n_y = 128$, and $n_t = 101$ points along the $x$, $y$, and $t$ dimensions, respectively.

\subsubsection{Navier-Stokes equations}


For the in-distribution tasks, 45 cases with varying geometries are generated as the training dataset $\mathbfcal{T}_{\rm training}$, where the number of inlet channels ranges from one to three. Among these, 40 cases are used as the NAS training dataset $\mathbfcal{T}_{\rm training\_NAS}$, and the remaining 5 cases as the NAS validation dataset $\mathbfcal{T}_{\rm validation\_NAS}$. An additional 20 test cases, with one to three inlet channels, are generated as $\mathbfcal{T}_{\rm test\_id}$. 

For the out-of-distribution test, new geometries with four inlet channels $\mathbfcal{T}_{\rm test\_ood}$ are generated.

The ground truth for velocity and pressure is computed using the high-fidelity CFD solver developed in \cite{Chiu2021DFIB,Chiu2023cDFIB}, which employs a finite volume-based dispersion-relation-preserving scheme for approximating the advection term and an improved differentially interpolated direct forcing immersed boundary method for handling irregular geometries. The spatial discretization in both the $x$ and $y$ dimensions is set to $n_x = 128$ and $n_y = 128$. The inlet BC is set as $u(x, y) = 1$ and $v(x, y) = 0$, while the outlet BC is $p(x, y) = 0$. No-slip conditions are imposed on the walls, meaning both $u(x, y)$ and $v(x, y)$ are zero along the wall boundaries. Neumann BC is applied to $p(x, y)$ at the walls. 

\subsubsection{Nonlinear heat equation}

For the nonlinear heat equation, the analytical solution is provided in Eq. (\ref{Eq:2D_heat_id}). The source term $\mathcal{S}(x, y, t)$ is derived to be self-consistent with the analytical solution. 

\begin{equation}
    u(x, y, t) = k_1 {\rm sin}(\pi x) {\rm exp}(-\gamma k_1 x^2) {\rm exp}(-\gamma t^2) + k_2 {\rm sin}(\pi y) {\rm exp}(-\gamma k_2 y^2) {\rm exp}(-\gamma t^2)
\label{Eq:2D_heat_id}
\end{equation}

For the in-distribution tasks, 40 sets of $\gamma, k_1, k_2$ are randomly sampled from the range $[1, \pi]$ to form the training dataset $\mathbfcal{T}_{\rm training}$, of which 36 cases are used as NAS training dataset $\mathbfcal{T}_{\rm training\_NAS}$, and 4 cases are used as NAS validation dataset $\mathbfcal{T}_{\rm validation\_NAS}$. An additional 10 sets of $\gamma, k_1, k_2$ sampled from the same range are used to construct the test dataset $\mathbfcal{T}_{\rm test\_id}$.

For the out-of-distribution evaluation, we design three modified source terms, each corresponding to a new analytical solution, denoted as F1, F2, and F3. These analytical solutions are given in Eqs. (\ref{Eq:2D_heat_F1})–(\ref{Eq:2D_heat_F3}). The same 10 sets of PDE parameters $\gamma, k_1, k_2$ from $\mathbfcal{T}_{\rm test\_id}$ are used to create new datasets for these three new problems, i.e., $\mathbfcal{T}_{{\rm test\_ood}\_F1}$, $\mathbfcal{T}_{{\rm test\_ood}\_F2}$, and $\mathbfcal{T}_{{\rm test\_ood}\_F3}$.

\begin{subequations}
\label{Eq:2D_heat_F1-3}
    \begin{numcases}{u(x, y, t) =}
        k_1 {\rm cos}(\pi x) {\rm exp}(-\gamma k_1 x^2) {\rm exp}(-\gamma t^2) + k_2 {\rm cos}(\pi y) {\rm exp}(-\gamma k_2 y^2) {\rm exp}(-\gamma t^2) \label{Eq:2D_heat_F1} \\ 
        k_1 k_2 {\rm sin}(\pi x) {\rm sin}(\pi y) {\rm exp}(-\gamma k_1 x^2) {\rm exp}(-\gamma k_2 y^2) {\rm exp}(-\gamma t^2) \label{Eq:2D_heat_F2} \\
        k_1 k_2 {\rm cos}(\pi x) {\rm cos}(\pi y) {\rm exp}(-\gamma k_1 x^2) {\rm exp}(-\gamma k_2 y^2) {\rm exp}(-\gamma t^2) \label{Eq:2D_heat_F3}
    \end{numcases}
\end{subequations}

The ground truth for $\mathbfcal{T}_{\rm training}$ and $\mathbfcal{T}_{\rm test\_id}$ is given by Eq. (\ref{Eq:2D_heat_id}), while the ground truth for $\mathbfcal{T}_{{\rm test\_ood}\_F1}$, $\mathbfcal{T}_{{\rm test\_ood}\_F2}$, and $\mathbfcal{T}_{{\rm test\_ood}\_F3}$ is provided by Eqs. (\ref{Eq:2D_heat_F1})-(\ref{Eq:2D_heat_F3}). The discretizations for $x$, $y$, and $t$ dimensions are $n_x = 64$, $n_y = 64$, and $n_t = 21$, respectively. Dirichlet BC is employed with values at the boundary determined based on ensuring self-consistency with the analytical solutions.

\subsection{NAS settings and training details}\label{subsec:training details}

\subsubsection{NAS settings}
For the generalized U-Net backbone of NOVA and NAS-U-Net, the number of downsampling and upsampling cells are set to four and three respectively, with each cell containing three intermediate nodes. The structures of all the downsampling or upsampling cells are kept the same. Consequently, the dimensionality of the encoding sequence $\boldsymbol{x}_{\rm encoding}$ is $2 \times 3 + 2 \times 3 = 12$. Since the rolling prediction method is adopted to predict the solution at each time step as described in Supplementary Section 6, the dimensionality of the NAS problem is 14. 

\subsubsection{Training details}

\hspace*{1.5em}\textbf{(1) Training details for NOVA}

For the nonlinear heat, flow mixing, diffusion-reaction, and Navier-Stokes equations, each architecture is trained for five epochs on the NAS training dataset $\mathbfcal{T}_{\rm training\_NAS}$ during the NAS phase to identify the architecture with the lowest RMSE on the NAS validation dataset $\mathbfcal{T}_{\rm validation\_NAS}$. After completing the physics-guided NAS, the best-performing architecture (i.e., the one achieving minimal RMSE on $\mathbfcal{T}_{\rm validation\_NAS}$) is re-trained on the combined dataset $\mathbfcal{T}_{\rm training\_NAS} \cup \mathbfcal{T}_{\rm validation\_NAS}$. Using five epochs with the Adam optimizer and a fixed learning rate of $5 \times 10^{-4}$, the optimized kernel parameters $\boldsymbol{\omega}^*$ in Eq. (\ref{Eq:training_kernel_formulation}) are learned. NOVA is then evaluated on $\mathbfcal{T}_{\rm test\_id}$ and $\mathbfcal{T}_{\rm test\_ood}$ by using the learned $\boldsymbol{\omega}^*$ together with the final layer parameters $\boldsymbol{\theta}^*$ as obtained through the physics-informed zero-shot adaptation procedure as per Eq. (\ref{Eq:theta}), as detailed in Algorithm \ref{algorithm_NAS} and Algorithm \ref{algorithm_test}. The convolutional kernel size is set to $5 \times 5$ for the  nonlinear heat, flow mixing, and diffusion-reaction equations, and to $7 \times 7$ for the Navier-Stokes equations. 

\textbf{(2) Training details for NAS-U-Net}

In contrast, the data-driven NAS-U-Net is trained for 20 epochs on the NAS training dataset $\mathbfcal{T}_{\rm training\_NAS}$ during the NAS phase to identify the architecture that achieves the lowest RMSE on the NAS validation dataset $\mathbfcal{T}_{\rm validation\_NAS}$. The selected architecture is then re-trained for 500 epochs on the combined dataset $\mathbfcal{T}_{\rm training\_NAS} \cup \mathbfcal{T}_{\rm validation\_NAS}$ to ensure model convergence. The NAS-U-Net is then evaluated on $\mathbfcal{T}_{\rm test\_id}$ and $\mathbfcal{T}_{\rm test\_ood}$ to obtain the final test results. Training is performed using the Adam optimizer with an initial learning rate of $5 \times 10^{-4}$. A reduce-on-plateau learning rate scheduler is applied, with a patience of 10 steps and a reduction factor of 0.7.

\textbf{(3) Training details for other models}

For the diffusion-reaction equations, the results of FNO, U-Net, and MLP-PINN are directly taken from \cite{takamoto2022pdebench}. In that study, FNO and U-Net are trained for 500 epochs using the Adam optimizer, with an initial learning rate of $1 \times 10^{-3}$, which is halved every 100 epochs. MLP-PINN is trained for 15000 epochs using the Adam optimizer with a fixed learning rate of $1 \times 10^{-3}$.

The maximum number of function evaluations ${\rm NFE_{max}}$ for NAS in both NOVA and NAS-U-Net is set to 200. Additionally, all experiments are implemented in the PyTorch framework and are run on an Intel Xeon Gold 6336Y 2.40GHz CPU workstation with a NVIDIA RTX A5000 GPU.

\subsection{NOVA for guided generative design}\label{subsec:guided diffusion}

\subsubsection{Guided diffusion process}\label{Guided diffusion process}

Diffusion models represent a powerful class of deep generative models, primarily used in machine learning for image generation, and have gained significant attention across various domains. The core concept of diffusion models is to simulate how data is gradually corrupted by noise and then learn how to reverse this process to generate new data samples. 

During the forward process, a small amount of Gaussian noise is added at each step, transforming the original data into pure noise through a Markov chain as per Eq. (\ref{diffusion process})

\begin{equation}
    q(\boldsymbol{x}_t | \boldsymbol{x}_{t-1}) = \mathcal{N}(\boldsymbol{x}_t; \; \sqrt{1 - \beta_t}\boldsymbol{x}_{t-1}, \; \beta_t\boldsymbol{I})
\label{diffusion process}
\end{equation}

\noindent where $\beta_t \in (0, 1)$ is a hyperparameter selected prior to model training. In contrast, the denoising process aims to progressively remove the noise as predicted through a neural network $\boldsymbol{\xi}$ to generate ``denoised'' samples as per Eq. (\ref{denoising process})

\begin{equation}
    p_{\boldsymbol{\xi}}(\boldsymbol{x}_{t-1} | \boldsymbol{x}_t) = \mathcal{N}(\boldsymbol{x}_{t-1}; \; \boldsymbol{\mu}_{\boldsymbol{\xi}}(\boldsymbol{x}_t), \; \boldsymbol{\Sigma}_{\boldsymbol{\xi}}(\boldsymbol{x}_t))
\label{denoising process}
\end{equation}

\noindent where $\boldsymbol{\mu}_{\boldsymbol{\xi}}(\boldsymbol{x}_t)$ and $\boldsymbol{\Sigma}_{\boldsymbol{\xi}}(\boldsymbol{x}_t)$ are the predicted mean and covariance of the denoising process. New samples can then be generated by sampling from $\mathcal{N}(0, 1)$ to undergo the denoising process according to Eq. (\ref{denoising process}).

\label{conditional denoising process}
In many AI for Science applications, novel designs (samples) are typically required, but these must satisfy specific performance criteria. Guided diffusion models provide a promising approach to guide the sampling process to meet these requirements \cite{dhariwal2021diffusion}. One effective way is to use the gradients from a differentiable performance evaluator to guide the pre-trained diffusion model \cite{maze2023diffusion}, as this eliminates the need to modify the diffusion model, which is a significant advantage given the advent of foundation models in Science. In the context of Navier-Stokes equations, a smaller pressure drop corresponds to smaller pressure gradients, leading to smoother and more stable flows without sudden accelerations or decelerations. Thus, we use the pressure drop between the inlet and outlet ($\Delta P$) as our simulated objective for the generative design process. The conditional denoising process is formulated in Eq. (\ref{conditional denoising process}) 

\begin{equation}
    \hat{p}_{\boldsymbol{\xi}}(\boldsymbol{x}_{t-1} | \boldsymbol{x}_t) = \mathcal{N}(\boldsymbol{x}_{t-1}; \; \boldsymbol{\mu}_{\boldsymbol{\xi}}(\boldsymbol{x}_t) - s\boldsymbol{\Sigma}_{\boldsymbol{\xi}}(\boldsymbol{x}_t) \nabla_{\boldsymbol{x}_t} c_{\boldsymbol{\omega}, \boldsymbol{\theta}}(\boldsymbol{x}_t), \; \boldsymbol{\Sigma}_{\boldsymbol{\xi}}(\boldsymbol{x}_t))
\label{conditional denoising process}
\end{equation}

\noindent where $s$ is the gradient scaling factor; $c_{\boldsymbol{\omega}, \boldsymbol{\theta}}(\boldsymbol{x}_t)$ is the predicted pressure drop between the inlet and outlet from the regressor, i.e., the pre-trained NOVA model described above; $\nabla_{\boldsymbol{x}_t} c_{\boldsymbol{\omega}, \boldsymbol{\theta}}(\boldsymbol{x}_t)$ denotes the gradient of $c_{\boldsymbol{\omega}, \boldsymbol{\theta}}(\boldsymbol{x}_t)$ with respect to $\boldsymbol{x}_t$. Similar to the unconditional denoising process, samples are first drawn from $\mathcal{N}(0, 1)$ and then progressively denoised according to Eq. (\ref{conditional denoising process}) to generate the conditional new data. 

\subsubsection{Training details of diffusion model}\label{Training details of diffusion model}

The input resolution for the diffusion model is $64 \times 64$. The training dataset consists of 20000 fluidic channel designs with varying numbers of inlet channels and is split into 80\% for training and 20\% for validation. The model is trained using the Adam optimizer for 60000 epochs with a fixed learning rate of $1 \times 10^{-4}$. Once trained, guidance (i.e., the pressure drop between the inlet and outlet predicted by the pre-trained NOVA) can be incorporated into the denoising process to generate conditional fluidic channel designs.

Since NOVA for Navier-Stokes equations was previously developed using $128 \times 128$ input resolution, the fluidic channel geometries generated by the guided diffusion process are upsampled to $128 \times 128$ for evaluation. The ground truth velocity and pressure fields for each fluidic channel geometry are obtained using the high-fidelity CFD solver developed in \cite{Chiu2021DFIB,Chiu2023cDFIB}.

\bmhead{Supplementary information}

Supplementary materials are provided in a separate file.

\bmhead{Acknowledgements}



This research is supported by the A*STAR Catalyst Project for Artificial Intelligence in Drug Discovery (AIDD) Programme.

\bibliography{sn-bibliography}

\end{document}


\title[Article Title]{Searching for Generalizable Neural Architectures for Computational Physics}


\author*[1,2]{\fnm{First} \sur{Author}}\email{iauthor@gmail.com}

\author[2,3]{\fnm{Second} \sur{Author}}\email{iiauthor@gmail.com}
\equalcont{These authors contributed equally to this work.}

\author[1,2]{\fnm{Third} \sur{Author}}\email{iiiauthor@gmail.com}
\equalcont{These authors contributed equally to this work.}

\affil*[1]{\orgdiv{Department}, \orgname{Organization}, \orgaddress{\street{Street}, \city{City}, \postcode{100190}, \state{State}, \country{Country}}}

\affil[2]{\orgdiv{Department}, \orgname{Organization}, \orgaddress{\street{Street}, \city{City}, \postcode{10587}, \state{State}, \country{Country}}}

\affil[3]{\orgdiv{Department}, \orgname{Organization}, \orgaddress{\street{Street}, \city{City}, \postcode{610101}, \state{State}, \country{Country}}}


\section*{\centering \Large Supplementary Information}
\vspace{0.5cm}

\section{Overall performance of NOVA across physical systems}\label{sec:overal performance}

The overall performance of NOVA across different physical systems is summarized in Supplementary Table \ref{tab:summary}.

\begin{table}[h]
\centering
\caption{Summary of NOVA performance across physical systems. Reported values are RMSEs for in-distribution and out-of-distribution tests (lower is better). “Gain” indicates the improvement factor computed as Baseline / NOVA.}
\label{tab:summary}
\begin{tabular}{lcccccc}
\toprule
Physical system & Test dataset & Baseline & Baseline RMSE & NOVA RMSE & \makecell{NOVA performance \\ gain ($\times$)} \\
\midrule
\multirow{5}{*}{\makecell{Diffusion-reaction\\equations}} 
 & $\mathbfcal{T}_{\rm test\_id}$  & FNO      & $8.1\!\times\!10^{-3}$ & $2.9\!\times\!10^{-3}$ & 2.8 \\
 & $\mathbfcal{T}_{\rm test\_id}$  & U\mbox{-}Net & $6.1\!\times\!10^{-2}$ & $2.9\!\times\!10^{-3}$ & 21.0 \\
 & $\mathbfcal{T}_{\rm test\_id}$  & MLP\mbox{-}PINN & $1.9\!\times\!10^{-1}$ & $2.9\!\times\!10^{-3}$ & 65.5 \\
 & $\mathbfcal{T}_{\rm test\_ood}$ & FNO      & $7.9\!\times\!10^{-1}$ & $4.6\!\times\!10^{-3}$ & 172 \\
 & $\mathbfcal{T}_{\rm test\_ood}$ & U\mbox{-}Net & $3.5\!\times\!10^{-1}$ & $4.6\!\times\!10^{-3}$ & 76 \\
\noalign{\vskip 2pt}
\noalign{\hrule height 0.4pt}
\noalign{\vskip 1pt}
\noalign{\hrule height 0.4pt}
\noalign{\vskip 2pt}
\multirow{2}{*}{\makecell{Navier-Stokes\\equations}}
 & $\mathbfcal{T}_{\rm test\_id}$  & NAS\mbox{-}U\mbox{-}Net & $2.1\!\times\!10^{-1}$ & $4.4\!\times\!10^{-2}$ & 4.8 \\
 & $\mathbfcal{T}_{\rm test\_ood}$ & NAS\mbox{-}U\mbox{-}Net & $5.1\!\times\!10^{-1}$ & $9.5\!\times\!10^{-2}$ & 5.4 \\
\noalign{\vskip 2pt}
\noalign{\hrule height 0.4pt}
\noalign{\vskip 1pt}
\noalign{\hrule height 0.4pt}
\noalign{\vskip 2pt}
\multirow{4}{*}{\makecell{Nonlinear heat\\equation}}
 & $\mathbfcal{T}_{\rm test\_id}$     & NAS\mbox{-}U\mbox{-}Net & $1.5\!\times\!10^{-1}$ & $3.0\!\times\!10^{-3}$ & 50 \\
 & $\mathbfcal{T}_{{\rm test\_ood}\_F1}$ & NAS\mbox{-}U\mbox{-}Net & $9.7\!\times\!10^{-1}$ & $3.0\!\times\!10^{-2}$ & 32 \\
 & $\mathbfcal{T}_{{\rm test\_ood}\_F2}$ & NAS\mbox{-}U\mbox{-}Net & $4.5\!\times\!10^{-1}$ & $2.6\!\times\!10^{-3}$ & 173 \\
 & $\mathbfcal{T}_{{\rm test\_ood}\_F3}$ & NAS\mbox{-}U\mbox{-}Net & $6.2\!\times\!10^{-1}$ & $4.6\!\times\!10^{-3}$ & 135 \\
\noalign{\vskip 2pt}
\noalign{\hrule height 0.4pt}
\noalign{\vskip 1pt}
\noalign{\hrule height 0.4pt}
\noalign{\vskip 2pt}
\makecell{\textbf{Overall}\\\textbf{(out-of-distribution)}} 
 & -- & -- & -- & -- &
\textbf{avg $\approx 99\times$} \\
\bottomrule
\end{tabular}
\end{table}

\clearpage

\section{NOVA performance for diffusion-reaction equations}\label{sec:diffusion-reaction}

The prediction results of NOVA, FNO, U-Net, and MLP-PINN for the 2D diffusion-reaction equations on $\mathbfcal{T}_{\rm test\_id}$ and $\mathbfcal{T}_{\rm test\_ood}$ are listed in Supplementary Table \ref{Table:2D_DR}. FNO and U-Net are purely data-driven models.




\begin{table}[h]
\caption{Test results for the 2D diffusion-reaction equations on $\mathbfcal{T}_{\rm test\_id}$ and $\mathbfcal{T}_{\rm test\_ood}$. 
The results of FNO, U-Net, and MLP-PINN on $\mathbfcal{T}_{\rm test\_id}$ are extracted from \cite{takamoto2022pdebench}.}
\label{Table:2D_DR}
\centering
\begin{tabular}{ccccccc}
\toprule
\multirow{2}{*}{Test dataset} & \multirow{2}{*}{Model} & \multicolumn{5}{c}{Metric} \\ 
\cmidrule(lr){3-7}
~ & ~ & RMSE & nRMSE & max error & cRMSE & bRMSE \\
\midrule
\multirow{4}{*}{$\mathbfcal{T}_{\rm test\_id}$} 
 & FNO & $8.1 \times 10^{-3}$ & $1.2 \times 10^{-1}$ & $9.1 \times 10^{-2}$ & $1.7 \times 10^{-3}$ & $2.7 \times 10^{-2}$ \\
 & U-Net & $6.1 \times 10^{-2}$ & $8.4 \times 10^{-1}$ & $1.9 \times 10^{-1}$ & $3.9 \times 10^{-2}$ & $7.8 \times 10^{-2}$ \\
 & MLP-PINN & $1.9 \times 10^{-1}$ & $1.6$ & $5.0 \times 10^{-1}$ & $1.3 \times 10^{-1}$ & $2.2 \times 10^{-1}$ \\
 & NOVA & $2.9 \times 10^{-3}$ & $3.4 \times 10^{-2}$ & $7.4 \times 10^{-2}$ & $1.5 \times 10^{-4}$ & $9.6 \times 10^{-3}$ \\
\noalign{\vskip 2pt}
\noalign{\hrule height 0.4pt}
\noalign{\vskip 1pt}
\noalign{\hrule height 0.4pt}
\noalign{\vskip 2pt}
\multirow{3}{*}{$\mathbfcal{T}_{\rm test\_ood}$} & FNO & $7.9 \times 10^{-1}$ & $3.1$ & $2.6$ & $4.3 \times 10^{-1}$ & $2.7 \times 10^{-1}$ \\
& U-Net & $3.5 \times 10^{-1}$ & $9.9 \times 10^{-1}$ & $6.2 \times 10^{-1}$ & $3.1 \times 10^{-1}$ & $2.9 \times 10^{-1}$\\
& NOVA & $4.6 \times 10^{-3}$ & $1.8 \times 10^{-2}$ & $1.8 \times 10^{-2}$ & $3.5 \times 10^{-3}$ & $4.4 \times 10^{-3}$ \\
 
\bottomrule
\end{tabular}
\end{table}

The ablation analysis results for the 2D diffusion-reaction equations on $\mathbfcal{T}_{\rm test\_id}$ and $\mathbfcal{T}_{\rm test\_ood}$ are presented in Supplementary Table \ref{ablation analysis}.
\break{}

\textbf{Effectiveness of physics-informed adaptation}. In addition to evaluating the effectiveness of a physics-guided NAS, we further select two architectures: (1) the one with the lowest training loss on $\mathbfcal{T}_{\rm training\_NAS}$, and (2) the final NOVA architecture with the lowest RMSE on $\mathbfcal{T}_{\rm validation\_NAS}$, and evaluate the contribution of the physics-informed adaptation. Instead of applying physics-informed adaptation, we re-train both architectures using only data loss for $5$ and $500$ epochs. The resulting models, denoted as NOVA-minimal-training-loss-data-driven-5, NOVA-minimal-training-loss-data-driven-500, NOVA-data-driven-5, and NOVA-data-driven-500, are then evaluated on PDE scenarios in $\mathbfcal{T}_{\rm test\_id}$. As shown in Supplementary Table \ref{ablation analysis}, NOVA surpasses both re-trained architectures with 5-step data-driven updates (i.e., NOVA-minimal-training-loss-data-driven-5 and NOVA-data-driven-5), improving performance by about two orders of magnitude over NOVA-data-driven-5 and one order of magnitude over NOVA-minimal-training-loss-data-driven-5. Notably, even after 500 steps of data-driven training, all metrics for both models remain markedly inferior to those of NOVA. These results highlight the critical importance of physics-informed zero-shot adaptation in improving generalization to unseen test tasks.

\begin{table}[h]
\caption{Ablation results for the 2D diffusion-reaction equations on $\mathbfcal{T}_{\rm test\_id}$ and $\mathbfcal{T}_{\rm test\_ood}$. ``NOVA-minimal-training-loss'' refers to the architecture with the lowest training loss on $\mathbfcal{T}_{\rm training\_NAS}$ among the 200 architectures in the NAS process, followed by physics-informed zero-shot adaptation on $\mathbfcal{T}_{\rm test\_id}$ or $\mathbfcal{T}_{\rm test\_ood}$. ``NOVA-median-training-loss'' refers to the architecture with the median training loss on $\mathbfcal{T}_{\rm training\_NAS}$ among the 200 architectures in the NAS process, also followed by physics-informed zero-shot adaptation on $\mathbfcal{T}_{\rm test\_id}$ or $\mathbfcal{T}_{\rm test\_ood}$. ``NOVA-minimal-training-loss-data-driven-5'' denotes the architecture with the lowest training loss on $\mathbfcal{T}_{\rm training\_NAS}$ among the 200 architectures in the NAS process, re-trained using five epochs with data loss only, and then evaluated on $\mathbfcal{T}_{\rm test\_id}$ or $\mathbfcal{T}_{\rm test\_ood}$ (without any physics-informed zero-shot adaptation). ``NOVA-minimal-training-loss-data-driven-500'' follows the same process but with 500 epochs. ``NOVA-data-driven-5'' denotes the architecture with the lowest RMSE on $\mathbfcal{T}_{\rm validation\_NAS}$ among the 200 architectures in the NAS process, re-trained using five epochs with data loss only, and then evaluated on $\mathbfcal{T}_{\rm test\_id}$ or $\mathbfcal{T}_{\rm test\_ood}$ (without any physics-informed zero-shot adaptation). ``NOVA-data-driven-500'' follows the same process but with 500 epochs. ``NOVA'' denotes the architecture with the lowest RMSE on $\mathbfcal{T}_{\rm validation\_NAS}$ among the 200 architectures in the NAS process, followed by physics-informed zero-shot adaptation on $\mathbfcal{T}_{\rm test\_id}$ or $\mathbfcal{T}_{\rm test\_ood}$.}
\label{ablation analysis}
\centering
\begin{tabular}{ccccccc}
\toprule
\multirow{2}{*}{Test Dataset} & \multirow{2}{*}{Model} & \multicolumn{5}{c}{Metric} \\ 
\cmidrule(lr){3-7}
~ & ~ & RMSE & nRMSE & max error & cRMSE & bRMSE \\
\midrule
\multirow{11}{*}{$\mathbfcal{T}_{\rm test\_id}$}
 & \makecell{NOVA-minimal-\\training-loss} & $2.1 \times 10^{-2}$ & $2.4 \times 10^{-1}$ & $1.4 \times 10^{-1}$ & $2.0 \times 10^{-4}$ & $4.0 \times 10^{-2}$ \\
 &  \makecell{NOVA-median-\\training-loss} & $7.8 \times 10^{-3}$ & $9.1 \times 10^{-2}$ & $8.9 \times 10^{-2}$ & $1.5 \times 10^{-4}$ & $1.9 \times 10^{-2}$ \\
 &  \makecell{NOVA-minimal-training-\\loss-data-driven-5} & $9.8 \times 10^{-2}$ & $1.7$ & $3.2 \times 10^{-1}$ & $2.2 \times 10^{-2}$ & $9.4 \times 10^{-2}$ \\
 & \makecell{NOVA-minimal-training-\\loss-data-driven-500} & $2.5 \times 10^{-2}$ & $2.7 \times 10^{-1}$ & $1.2 \times 10^{-1}$ & $4.9 \times 10^{-3}$ & $3.2 \times 10^{-2}$ \\
 & NOVA-data-driven-5 & $2.4 \times 10^{-1}$ & $4.5$ & $2.6$ & $2.0 \times 10^{-2}$ & $2.3 \times 10^{-1}$ \\
 & NOVA-data-driven-500 & $5.2 \times 10^{-2}$ & $6.2 \times 10^{-1}$ & $7.4 \times 10^{-1}$ & $1.8 \times 10^{-2}$ & $1.1 \times 10^{-1}$ \\
 & NOVA & $2.9 \times 10^{-3}$ & $3.4 \times 10^{-2}$ & $7.4 \times 10^{-2}$ & $1.5 \times 10^{-4}$ & $9.6 \times 10^{-3}$ \\
\noalign{\vskip 2pt}
\noalign{\hrule height 0.4pt}
\noalign{\vskip 1pt}
\noalign{\hrule height 0.4pt}
\noalign{\vskip 2pt}
\multirow{11}{*}{$\mathbfcal{T}_{\rm test\_ood}$}
 & \makecell{NOVA-minimal-\\training-loss} & $9.2 \times 10^{-3}$ & $3.6 \times 10^{-2}$ & $1.0 \times 10^{-1}$ & $3.7 \times 10^{-3}$ & $1.8 \times 10^{-2}$ \\
 & \makecell{NOVA-median-\\training-loss} & $6.0 \times 10^{-3}$ & $2.5 \times 10^{-2}$ & $8.5 \times 10^{-2}$ & $3.8 \times 10^{-3}$ & $1.3 \times 10^{-2}$ \\
 & \makecell{NOVA-minimal-training-\\loss-data-driven-5} & $3.2 \times 10^{-1}$ & $9.1 \times 10^{-1}$ & $6.4 \times 10^{-1}$ & $2.5 \times 10^{-1}$ & $2.3 \times 10^{-1}$ \\
 & \makecell{NOVA-minimal-training-\\loss-data-driven-500} & $3.9 \times 10^{-1}$ & $1.2$ & $8.2 \times 10^{-1}$ & $3.3 \times 10^{-1}$ & $3.1 \times 10^{-1}$ \\
 & NOVA-data-driven-5 & $3.9 \times 10^{-1}$ & $1.2$ & $3.1$ & $2.5 \times 10^{-1}$ & $3.0 \times 10^{-1}$ \\
 & NOVA-data-driven-500 & $3.1 \times 10^{-1}$ & $8.8 \times 10^{-1}$ & $5.3 \times 10^{-1}$ & $2.7 \times 10^{-1}$ & $2.6 \times 10^{-1}$ \\
 & NOVA & $4.6 \times 10^{-3}$ & $1.8 \times 10^{-2}$ & $1.8 \times 10^{-2}$ & $3.5 \times 10^{-3}$ & $4.4 \times 10^{-3}$ \\
\bottomrule
\end{tabular}
\end{table}



\clearpage

\section{NOVA performance for Navier-Stokes equations}\label{sec:N-S}

The prediction results of NOVA and NAS-U-Net for the 2D steady-state incompressible Navier-Stokes equations on $\mathbfcal{T}_{\rm test\_id}$ and $\mathbfcal{T}_{\rm test\_ood}$ are compared in Supplementary Table \ref{Table:2D_NS}.


\begin{table}[h]
\caption{Test results for the 2D steady-state incompressible Navier-Stokes equations on $\mathbfcal{T}_{\rm test\_id}$ and $\mathbfcal{T}_{\rm test\_ood}$.}
\label{Table:2D_NS}
\centering
\begin{tabular}{ccccccc}
\toprule
\multirow{2}{*}{Test dataset} & \multirow{2}{*}{Model} & \multicolumn{5}{c}{Metric} \\ 
\cmidrule(lr){3-7}
~ & ~ & RMSE & nRMSE & max error & cRMSE & bRMSE \\
\midrule
\multirow{2}{*}{$\mathbfcal{T}_{\rm test\_id}$}
 & NAS-U-Net & $2.1 \times 10^{-1}$ & $8.2 \times 10^{-1}$ & $1.6$ & $3.1 \times 10^{-2}$ & $2.0 \times 10^{-1}$ \\
 & NOVA & $4.4 \times 10^{-2}$ & $1.8 \times 10^{-1}$ & $1.0$ & $4.9 \times 10^{-3}$ & $5.8 \times 10^{-2}$ \\
\noalign{\vskip 2pt}
\noalign{\hrule height 0.4pt}
\noalign{\vskip 1pt}
\noalign{\hrule height 0.4pt}
\noalign{\vskip 2pt}
\multirow{2}{*}{$\mathbfcal{T}_{\rm test\_ood}$}
 & NAS-U-Net & $5.1 \times 10^{-1}$ & $7.7 \times 10^{-1}$ & $2.7$ & $1.7 \times 10^{-1}$ & $4.5 \times 10^{-1}$ \\
 & NOVA & $9.5 \times 10^{-2}$ & $1.8 \times 10^{-1}$ & $1.9$ & $6.8 \times 10^{-3}$ & $1.2 \times 10^{-1}$ \\
\bottomrule
\end{tabular}
\end{table}

\clearpage

\section{NOVA performance for nonlinear heat equation}\label{sec:heat}

The prediction results of NOVA and NAS-U-Net for the 2D nonlinear heat equation on $\mathbfcal{T}_{\rm test\_id}$, $\mathbfcal{T}_{{\rm test\_ood}\_F1}$, $\mathbfcal{T}_{{\rm test\_ood}\_F2}$, and $\mathbfcal{T}_{{\rm test\_ood}\_F3}$ are compared in Supplementary Table \ref{Table:2D_heat_OOD}.

\begin{table}[h]
\caption{Test results for the 2D nonlinear heat equation on $\mathbfcal{T}_{\rm test\_id}$, $\mathbfcal{T}_{{\rm test\_ood}\_F1}$, $\mathbfcal{T}_{{\rm test\_ood}\_F2}$, and $\mathbfcal{T}_{{\rm test\_ood}\_F3}$.}
\label{Table:2D_heat_OOD}
    \begin{tabular}{ccccccc}
    \toprule
    \multirow{2}{*}{Test dataset} & \multirow{2}{*}{Model} & \multicolumn{5}{c}{Metric} \\ \cmidrule(lr){3-7}
    ~ & ~ & RMSE & nRMSE & max error & cRMSE & bRMSE \\
    \midrule
    \multirow{2}{*}{$\mathbfcal{T}_{\rm test\_id}$} & NAS-U-Net & $1.5 \times 10^{-1}$ & $3.4 \times 10^{-1}$ & $4.1 \times 10^{-1}$ & $8.6 \times 10^{-3}$ & $1.1 \times 10^{-1}$ \\
    ~ & NOVA & $3.0 \times 10^{-3}$ & $6.9 \times 10^{-3}$ & $3.7 \times 10^{-2}$ & $7.4 \times 10^{-4}$ & $6.5 \times 10^{-3}$ \\
    \noalign{\vskip 2pt}
    \noalign{\hrule height 0.4pt}
    \noalign{\vskip 1pt}
    \noalign{\hrule height 0.4pt}
    \noalign{\vskip 2pt}
    \multirow{2}{*}{$\mathbfcal{T}_{{\rm test\_ood}\_F1}$} & NAS-U-Net & $9.7 \times 10^{-1}$ & $1.4$ & $2.5$ & $5.3 \times 10^{-1}$ & $8.6 \times 10^{-1}$ \\ ~ & NOVA & $3.0 \times 10^{-2}$ & $5.6 \times 10^{-2}$ & $1.5 \times 10^{-1}$ & $3.0 \times 10^{-2}$ & $3.7 \times 10^{-2}$ \\
    \midrule
    \multirow{2}{*}{$\mathbfcal{T}_{{\rm test\_ood}\_F2}$} & NAS-U-Net & $4.5 \times 10^{-1}$ & $1.6$ & $1.4$ & $2.6 \times 10^{-2}$ & $3.0 \times 10^{-1}$ \\ ~ & NOVA & $2.6 \times 10^{-3}$ & $2.0 \times 10^{-2}$ & $3.9 \times 10^{-2}$ & $7.5 \times 10^{-4}$ & $7.4 \times 10^{-3}$ \\
    \midrule
    \multirow{2}{*}{$\mathbfcal{T}_{{\rm test\_ood}\_F3}$} & NAS-U-Net & $6.2 \times 10^{-1}$ & $1.2$ & $2.7$ & $1.5 \times 10^{-1}$ & $2.4 \times 10^{-1}$ \\ ~ & NOVA & $4.6 \times 10^{-3}$ & $1.5 \times 10^{-2}$ & $2.1 \times 10^{-2}$  & $4.2 \times 10^{-3}$ & $4.8 \times 10^{-3}$ \\
    \bottomrule
    \end{tabular}
\end{table}


\clearpage

\section{NOVA performance for flow mixing equation}\label{sec:flow-mixing}

In this section, we further confirm NOVA's capability for generalization in the extreme scenario where there is a complete absence of data, highlighting its utility in scenarios where data is costly to obtain. A 2D transient flow mixing equation is used to illustrate how NOVA can be implemented even in the complete absence of data, while being a distinctly different physical system from the experiments in the main text. 

The governing PDE of the 2D flow mixing equation is

\begin{equation}
\label{Eq:2D_FM}
\begin{aligned}
& \frac{\partial u(x, y, t)}{\partial t} - \frac{v_t \cdot y}{v_{t\_{\rm max}} \cdot r} \frac{\partial u(x, y, t)}{\partial x} + \frac{v_t \cdot x}{v_{t\_{\rm max}} \cdot r} \frac{\partial u(x, y, t)}{\partial y} = 0, \\
& v_t = {\rm sech}^2(r) {\rm tanh}(r), \ r = \sqrt{x^2 + y^2}, \\
& \hspace{7cm} x \in [-4, 4], \; y \in [-4, 4], \; t \in [0, 4]
\end{aligned}
\end{equation}

\noindent where the PDE coefficient $v_{t\_{\rm max}}$ is varied to construct a range of PDE scenarios.

\subsection{Data-free NOVA} 

Even without any data information, NOVA can still be trained using the physics loss $\mathcal{L}_{\rm Physics}$, as expressed in Supplementary Eq. (\ref{Eq:physics loss})

\begin{equation}
\begin{aligned}
    & \mathcal{L}_{\rm Physics} = \mathcal{L}_{\rm PDE} + \lambda_{\rm BC}\mathcal{L}_{\rm BC} + \lambda_{\rm IC}\mathcal{L}_{\rm IC} \\
    & \mathcal{L}_{\rm PDE} = \frac{1}{n_{\rm PDE}} \sum_{i = 1}^{n_{\rm PDE}} \big\lVert \mathcal{N}[\boldsymbol{u}(\boldsymbol{x}_i, t_i)] - \mathcal{S}(\boldsymbol{x}_i, t_i) \big\rVert_2^{2} \\
    & \mathcal{L}_{\rm BC} = \frac{1}{n_{\rm BC}} \sum_{i = 1}^{n_{\rm BC}} \big\lVert \mathcal{B}[\boldsymbol{u}(\boldsymbol{x}_i, t_i)] - \boldsymbol{g}(\boldsymbol{x}_i, t_i) \big\rVert_2^{2} \\
    & \mathcal{L}_{\rm IC} = \frac{1}{n_{\rm IC}} \sum_{i = 1}^{n_{\rm IC}} \big\lVert \boldsymbol{u}(\boldsymbol{x}_i, 0) - \boldsymbol{u}_0(\boldsymbol{x}_i) \big\rVert_2^{2}
\label{Eq:physics loss}
\end{aligned}
\end{equation}

\noindent where $\lambda_{\rm BC}$ and $\lambda_{\rm IC}$ are the scaling factors to enforce BCs and ICs, respectively. The PDE domain is given by $\boldsymbol{\Omega} \times (0, t_m]$, the BC domain by $\partial \boldsymbol{\Omega} \times (0, t_m]$, and the IC domain by $\boldsymbol{\Omega}$.

In this case, NOVA's kernel parameters are trained through the following update step.

\begin{equation}
\boldsymbol{\omega}^* = \boldsymbol{\omega}^{(0)}-\sum_{i = 0}^{N-1} \eta \nabla_{\boldsymbol{\omega}}\mathcal L_{\rm Physics}(\boldsymbol{\omega}^{(i)}, {\rm PDE}, {\rm BC}, {\rm IC})
\label{Eq:training_formulation}
\end{equation}

A scaling factor of $\lambda_{\rm PDE} = 0.01$ is applied to enforce the Dirichlet BC for the data-free NOVA experiment.

\subsection{PDE scenario data}

For the flow mixing equation, the analytical solution $u(x, y, t)$ is as per Supplementary Eq. (\ref{Eq:2D_FM_analytical}). 

\begin{equation}
\label{Eq:2D_FM_analytical}
\begin{aligned}
& u(x, y, t) = -{\rm tanh}[\frac{y}{2}{\rm cos}(\omega t) - \frac{x}{2}{\rm sin}(\omega t)], \ \omega = \frac{v_t}{r \cdot v_{t\_{\rm max}}}
\end{aligned}
\end{equation}

The ground truth for $\mathbfcal{T}_{\rm training}$, $\mathbfcal{T}_{\rm test\_id}$ and $\mathbfcal{T}_{\rm test\_ood}$ is obtained using Supplementary Eq. (\ref{Eq:2D_FM_analytical}). The discretizations along the $x$, $y$, and $t$ dimensions are set to $n_x = 128$, $n_y = 128$, and $n_t = 21$, respectively. The Dirichlet BC is employed with values at the boundary determined based on ensuring self-consistency with the analytical solutions.

For this problem, the in-distribution and out-of-distribution tasks ($\mathbfcal{T}_{\rm test\_id}$ and $\mathbfcal{T}_{\rm test\_ood}$) are defined based on the range of the PDE coefficient $v_{t\_{\rm max}}$. 

For the in-distribution tasks, 40 values of $v_{t\_{\rm max}} \in [0.2, 0.5]$ are randomly selected to form the training dataset $\mathbfcal{T}_{\rm training}$, which is further divided into 36 cases for NAS training $\mathbfcal{T}_{\rm training\_NAS}$ and 4 cases for NAS validation $\mathbfcal{T}_{\rm validation\_NAS}$. An additional 10 in-distribution tasks are randomly sampled  (i.e., $v_{t\_{\rm max}} \in [0.2, 0.5]$) to construct the test dataset $\mathbfcal{T}_{\rm test\_id}$. 

For the out-of-distribution experiments, another 10 tasks are randomly sampled with $v_{t\_{\rm max}} \in [0.1, 0.2]$ to form a new test dataset $\mathbfcal{T}_{\rm test\_ood}$.



\subsection{Results for NOVA and NAS-U-Net}


As shown in Supplementary Table \ref{Table:2D_flow_mixing} and Supplementary Fig. \ref{Fig:2D_flow_mixing_gp}a, even without any training data, NOVA still outperforms the data-driven NAS-U-Net on both global and local metrics. Supplementary Fig. \ref{Fig:2D_flow_mixing_gp}c illustrates that NOVA's predictions at each time step are more accurate, ensuring that subsequent steps can be reliably predicted using an autoregressive rollout.



Supplementary Table \ref{Table:2D_flow_mixing} and Fig. \ref{Fig:2D_flow_mixing_gp}b illustrate how NOVA can still achieve accurate predictions by utilizing underlying physical laws on new tasks with out-of-distribution PDE coefficients, i.e., $\mathbfcal{T}_{\rm test\_ood}$. In contrast, the prediction performance of NAS-U-Net deteriorates greatly when applied to out-of-distribution problems. This highlights NOVA's effectiveness in addressing practical science and engineering problems characterized by limited data availability.


\begin{table}[h]
\caption{Test results for the 2D flow mixing equation on $\mathbfcal{T}_{\rm test\_id}$ and $\mathbfcal{T}_{\rm test\_ood}$.}
\label{Table:2D_flow_mixing}
\centering
\begin{tabular}{ccccccc}
\toprule
\multirow{2}{*}{Test dataset} & \multirow{2}{*}{Model} & \multicolumn{5}{c}{Metric} \\ 
\cmidrule(lr){3-7}
~ & ~ & RMSE & nRMSE & max error & cRMSE & bRMSE \\
\midrule
\multirow{2}{*}{$\mathbfcal{T}_{\rm test\_id}$} 
 & NAS-U-Net & $7.1 \times 10^{-2}$ & $9.9 \times 10^{-2}$ & $3.5 \times 10^{-1}$ & $4.2 \times 10^{-3}$ & $9.5 \times 10^{-3}$ \\
 & NOVA       & $1.9 \times 10^{-2}$ & $2.6 \times 10^{-2}$ & $1.3 \times 10^{-1}$ & $1.4 \times 10^{-5}$ & $2.6 \times 10^{-5}$ \\
\noalign{\vskip 2pt}
\noalign{\hrule height 0.4pt}
\noalign{\vskip 1pt}
\noalign{\hrule height 0.4pt}
\noalign{\vskip 2pt}
\multirow{2}{*}{$\mathbfcal{T}_{\rm test\_ood}$} 
 & NAS-U-Net & $1.3 \times 10^{-1}$ & $1.8 \times 10^{-1}$ & $5.4 \times 10^{-1}$ & $9.8 \times 10^{-3}$ & $1.7 \times 10^{-2}$ \\
& NOVA & $5.1 \times 10^{-2}$ & $7.1 \times 10^{-2}$ & $3.0 \times 10^{-1}$ & $3.8 \times 10^{-6}$ & $9.0 \times 10^{-5}$ \\
\bottomrule
\end{tabular}
\end{table}

\begin{figure}[htp]
\centering
\vspace*{0mm}
\includegraphics[width=1\textwidth]{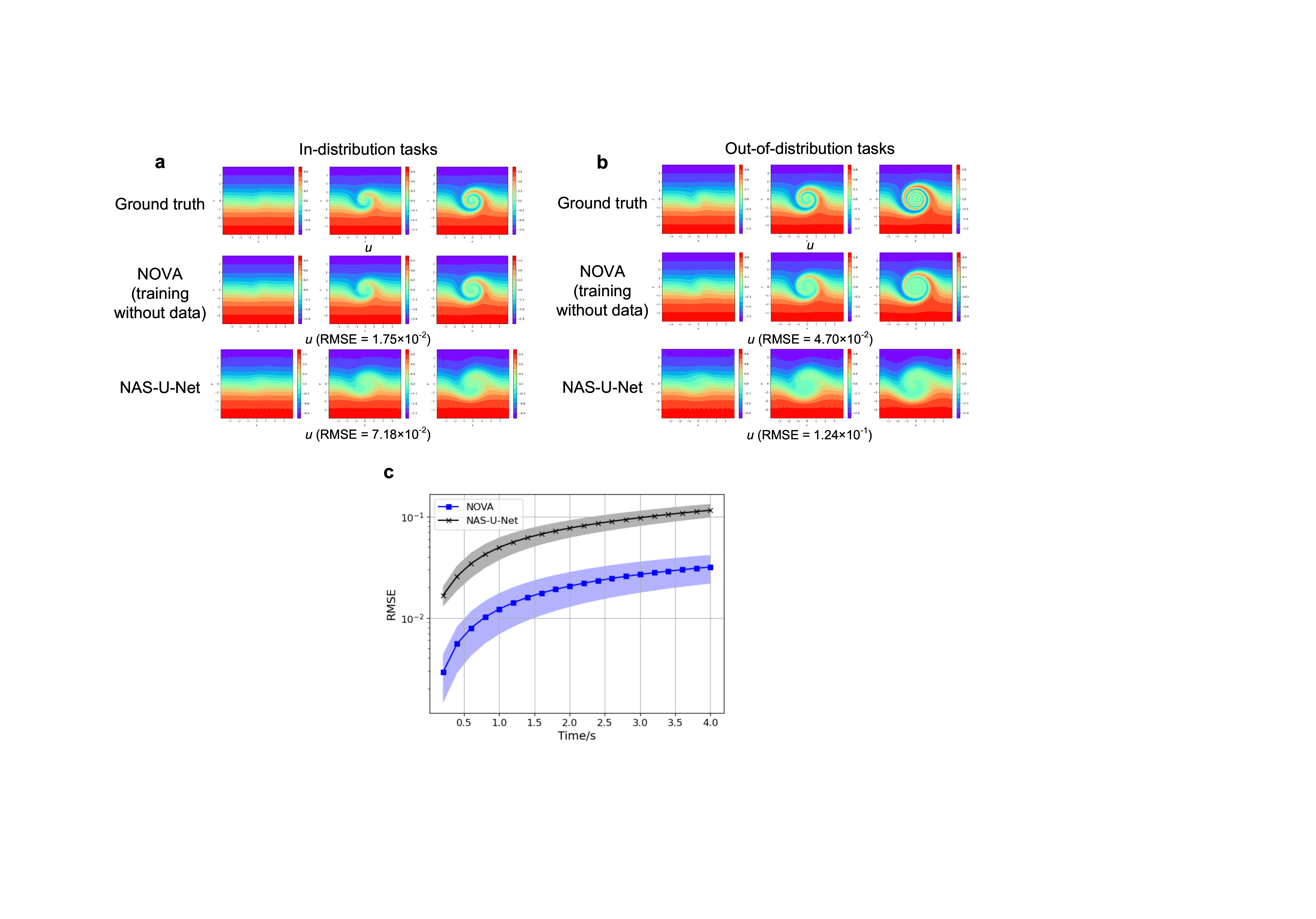}
\vspace*{-0mm}
\caption{Performance of NOVA for the 2D flow mixing equation. \textbf{a-b}. Prediction results of a representative case in (a) $\mathbfcal{T}_{\rm test\_id}$ and (b) $\mathbfcal{T}_{\rm test\_ood}$ for NOVA and NAS-U-Net. Each subfigure displays predictions at $t = 0.2$, $t = 2$, and $t = 4$ going from left to right. RMSE values in brackets are averages for the entire simulation time. \textbf{c}. Mean RMSE at each time step across tasks in $\mathbfcal{T}_{\rm test\_id}$ as-predicted by NOVA and NAS-U-Net. Shaded area indicates the standard deviation.}
\vspace*{0mm}
\label{Fig:2D_flow_mixing_gp}
\end{figure}


\clearpage

\section{Spatial and temporal derivatives}\label{sec:Spatial and temporal derivatives}

In recent studies, data-driven discretization methods have been explored for computing spatial derivatives, where neural networks learn optimal local operators directly from data \cite{bar2019learning}. In this work, however, we adopt finite difference schemes instead of data-driven approaches, owing to their simplicity and low computational overhead. Nonetheless, data-driven discretizations can also be seamlessly integrated with physics-informed constraints in NOVA and can be jointly learned in future work. 

Within CNN architectures, finite difference schemes are commonly used to calculate spatial derivatives based on the input grids \cite{CHIU2022114909}, as illustrated in Supplementary Fig. \ref{Fig:Spatial and temporal derivatives}a. To compute the derivative terms, it is necessary to utilize the surrounding values in the computational domains shifted by $\Delta x$ or $\Delta y$ along $+x$, $-x$, $+y$, and $-y$ axes,
where $\Delta x$ and $\Delta y$ are the spatial distance across evaluation points.




For transient problems, CNN needs to predict the solution at each time step. Since predicting the solution for the entire space-time domain at once is challenging, an autoregressive rollout is commonly employed \cite{krishnapriyan2021characterizing}, as depicted in Supplementary Fig. \ref{Fig:Spatial and temporal derivatives}b. This method starts with the initial condition $u(x, y, t_0)$ to predict the solution of time $t_1 = t_0 + \Delta t$, i.e., $u(x, y, t_1)$. The predicted value $u(x, y, t_1)$ is then used to predict the solution of the subsequent time $u(x, y, t_2)$ ($t_2 = t_1 + \Delta t$). The above procedures are repeated until the prediction for the final time $u(x, y, t_n)$ is obtained. Thus, the network can predict the entire state space using only the initial condition. The time derivative at each time step can also be approximated through finite difference schemes.




\begin{figure}[htp]
\centering
\vspace*{0mm}
\includegraphics[width=1\textwidth]{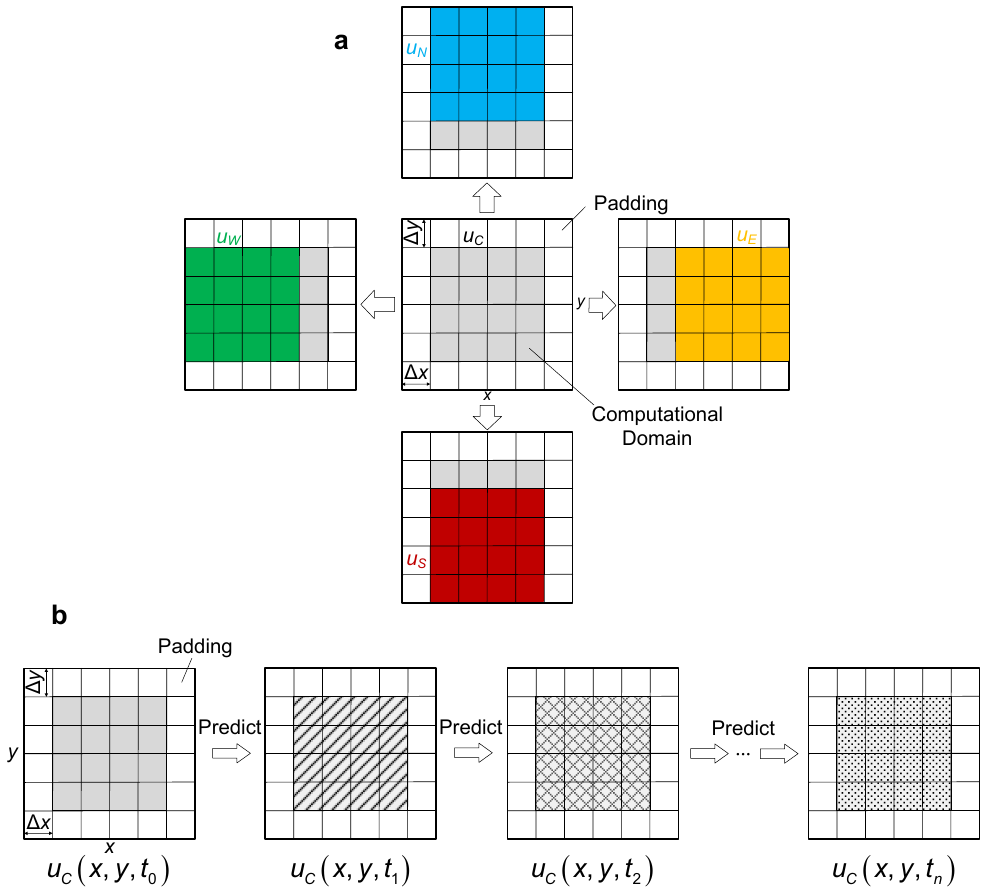}
\vspace*{-0mm}
\caption{Spatial and temporal derivatives of CNN architecture. \textbf{a}. Spatial derivatives. \textbf{b}. Temporal derivatives.}
\vspace*{0mm}
\label{Fig:Spatial and temporal derivatives}
\end{figure}

\clearpage

\section{Generalized U-Net architecture}\label{sec:Generalized U-Net}

The generalized U-Net structure is graphically expressed in Supplementary Fig. \ref{Fig:generalized U-Net}. 

\begin{figure}[htp]
\centering
\vspace*{0mm}
\includegraphics[width=1.0\textwidth]{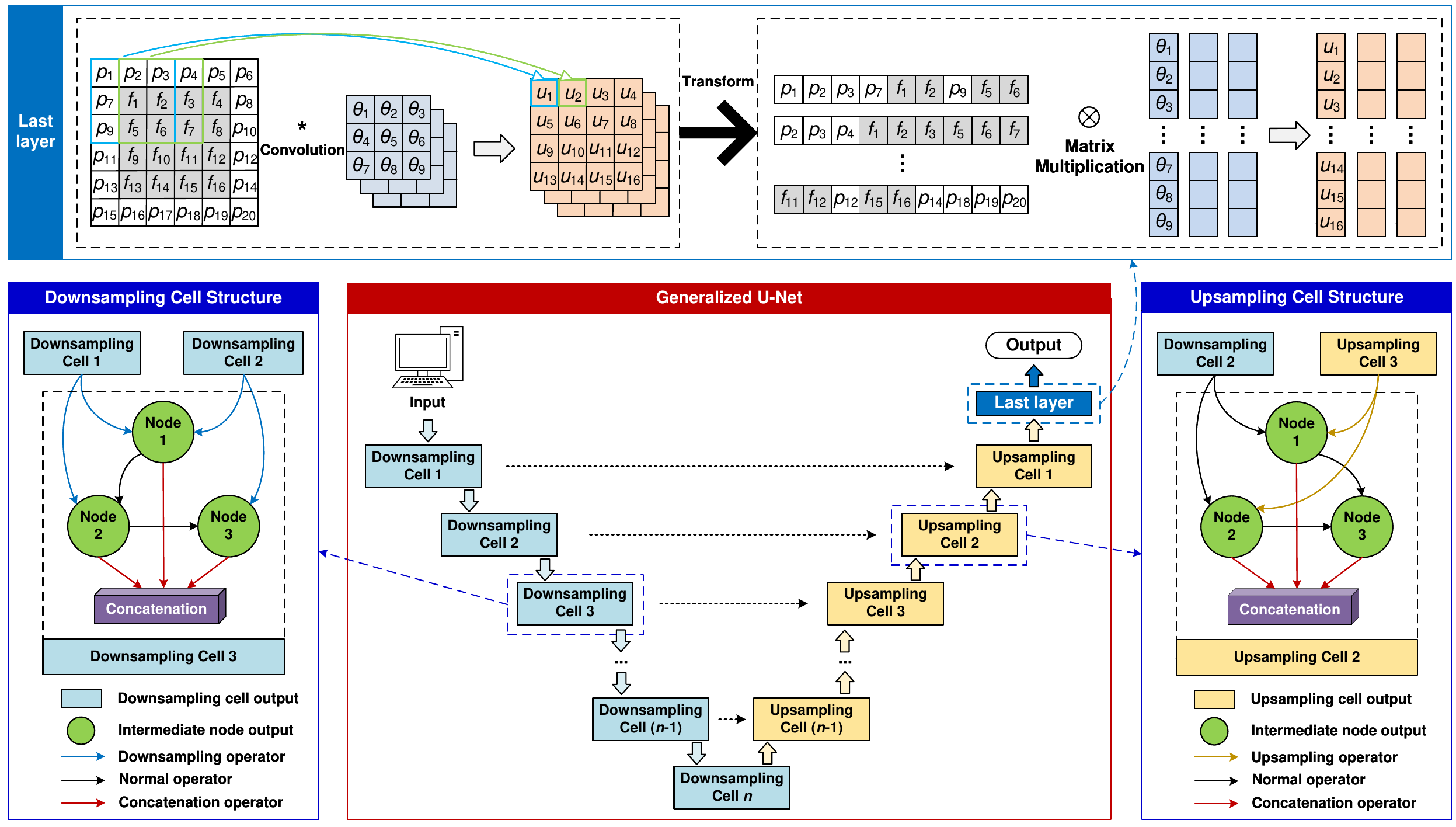}
\vspace*{-0mm}
\caption{Structure of the generalized U-Net.}
\vspace*{0mm}
\label{Fig:generalized U-Net}
\end{figure}

As illustrated in Supplementary Fig. \ref{Fig:generalized U-Net}, each downsampling cell includes two kinds of operators, i.e., downsampling operators and normal operators. Downsampling operators halve the number of pixels of the input while normal operators maintain the pixel of feature map as the input. Similarly, each upsampling cell includes upsampling operators and normal operators where the upsampling operators double the input pixel count. The downsampling, upsampling, and normal operators in the generalized U-Net are listed in Supplementary Table \ref{Table:operators}.

\begin{table}[h]
\renewcommand{\arraystretch}{1.5}
\caption{Downsampling, upsampling, and normal
operators.}
\begin{tabular}{@{}cc@{}}
\toprule
Type & Operator \\
\midrule
\makecell[c]{Downsampling operator \\ (stride = 2)} & \makecell[c]{3×3 squeeze-and-excitation, 3×3 dilated convolution, \\ 3×3 depthwise-separable convolution, 3×3 convolution, \\ 2×2 average pooling, 2×2 max pooling} \\ [5mm]
\makecell[c]{Upsampling operator \\ (stride = 2)} & \makecell[c]{3×3 transposed squeeze-and-excitation, \\ 3×3 transposed dilated convolution, 3×3 transposed convolution \\ 3×3 transposed depthwise-separable convolution} \\ [5mm]
\makecell[c]{Normal operator \\ (stride = 1)} & \makecell[c]{Identity, 3×3 squeeze-and-excitation, 3×3 dilated convolution, \\ 3×3 depthwise-separable convolution, 3×3 convolution}  \\
\botrule
\label{Table:operators}
\end{tabular}
\end{table}

\clearpage




\section{Evaluation metrics}\label{sec:evaluation metrics}

In this work, we use both global and local metrics to comprehensively evaluate the model performance for regression problems. This includes three classic global metrics: root mean squared error (RMSE), normalized RMSE (nRMSE), and maximum error (max error), which can be calculated by Supplementary Eq. (\ref{Eq:global metrics}) \cite{takamoto2022pdebench}

\begin{equation}
\label{Eq:global metrics}
\begin{aligned}
    & {\rm RMSE} = \sqrt{\frac{1}{n} \sum_{i=1}^{n} (u_{{\rm pred}, \, i} - u_{{\rm true}, \, i})^2} \\
    & {\rm nRMSE} = \frac{\sqrt{\frac{1}{n} \sum_{i=1}^{n} (u_{{\rm pred}, \, i} - u_{{\rm true}, \, i})^2}}{\sqrt{{\frac{1}{n} \sum_{i=1}^{n} u_{{\rm true}, \, i}^2}}} \\
    & \text{max error} = \max_{i=1, \ldots, n} \left| u_{{\rm pred}, \, i} - u_{{\rm true}, \, i} \right|
\end{aligned}
\end{equation}

\noindent where $u_{{\rm pred}, \, i}$ and $u_{{\rm true}, \, i}$ are the prediction value and ground truth of each point in the computational domain; $n$ is the number of points in the computational domain. The normalized RMSE can provide the scale-independent information, and the maximum error measures the worst prediction of the model.

Since these are global metrics, two other local metrics, i.e., RMSE of the conserved value (cRMSE) and RMSE at boundaries (bRMSE), are used to quantify specific failure modes as formulated in Supplementary Eq. (\ref{Eq:local metrics}) \cite{takamoto2022pdebench}

\begin{equation}
\label{Eq:local metrics}
\begin{aligned}
    & {\rm cRMSE} = \sqrt{\frac{1}{n} (\sum_{i=1}^{n} u_{{\rm pred}, \, i} - \sum_{i=1}^{n} u_{{\rm true}, \, i})^2} \\
    & {\rm bRMSE} = \sqrt{\frac{1}{n_{\rm BC}} \sum_{i=1}^{n_{\rm BC}} (u_{{\rm pred}, \, i} - u_{{\rm true}, \, i})^2}
\end{aligned}
\end{equation}

\noindent where $n_{\rm BC}$ is the number of BC points. cRMSE measures how well the model adheres to the physically conserved values, while bRMSE assesses the model's ability to accurately learn and predict at the boundaries.

\clearpage

\section{Theoretical analysis for lightweight training}

\noindent\textbf{Notation.} Let $\vect{\omega}\in\mathbb{R}^{q}$ denote the parameters of the network up to the penultimate layer, and
$\vect{\theta}\in \mathbb{R}^{p}$ the last-layer linear head.
For a fixed collocation/sampling set associated with task $\cT$,
$\mat{\Phi}(\vect{\omega})\in\mathbb{R}^{m\times p}$ is the design matrix obtained by stacking penultimate-layer features;
$\AT{\vect{\omega}}\in\mathbb{R}^{n\times p}$ and $\bT\in\mathbb{R}^{n}$ are the physics linearization matrix and right-hand side used in Eq.~(11).
We use $\|\cdot\|$ for the Euclidean norm on vectors and the spectral/operator norm on matrices;
$\sigma_{\min}(\cdot)$ is the smallest singular value; 
$\mat{P}_{\mat{M}}:=\mat{M}\mat{M}^{+}$ denotes the orthogonal projector onto $\mathrm{col}(\mat{M})$ via the pseudo-inverse~$\mat{M}^{+}$.
“Gradient descent (GD)” means the full-batch update
\[
\vect{\omega}^{(k+1)}=\vect{\omega}^{(k)}-\eta\,\nabla \cL(\vect{\omega}^{(k)}),
\quad k=0,1,\dots,K-1 .
\]
“Lightweight training” refers to using a small, fixed number \(K\) of GD steps, then freezing \(\vect{\omega}\).

\begin{assumption}[Smoothness, step size, feature/physics Lipschitz]\label{ass:main}
\leavevmode
\begin{itemize}
  \item[A1.] The objective $\cL(\vect{\omega})$ (either $\cL_{\mathrm{Data}}$ from Eq.~(5) or $\cL_{\mathrm{Physics}}$ from Supplementary Eq.~(2)) is $L_{\mathrm{sm}}$-smooth:
  \[
  \|\nabla \cL(\vect{\omega})-\nabla \cL(\vect{\omega}')\|\le L_{\mathrm{sm}}\,\|\vect{\omega}-\vect{\omega}'\|.
  \]
  \item[A2.] $\cL(\vect{\omega})$ is bounded below by its optimal value $\cL^{\ast}$.
  \item[A3.] GD in Eq.~(7) or Supplementary Eq.~(3) uses a fixed step size $\eta\in(0,\,1/L_{\mathrm{sm}}]$:
  \(
  \vect{\omega}^{(k+1)}=\vect{\omega}^{(k)}-\eta\,\nabla \cL(\vect{\omega}^{(k)}).
  \)
  \item[A4.] Penultimate features are $G$-Lipschitz in $\vect{\omega}$; letting $\mat{\Phi}(\vect{\omega})$ be the design matrix,
  \(
  \|\mat{\Phi}(\vect{\omega})-\mat{\Phi}(\vect{\omega}')\|\le G\,\|\vect{\omega}-\vect{\omega}'\|.
  \)
  \item[A5.] The physics linearization that builds $\AT{\vect{\omega}}$ in Eq.~(11) is $\beta$-Lipschitz w.r.t.\ features, i.e.,
\(
\|\AT{\vect{\omega}}-\AT{\vect{\omega}'}\| \le \beta\,\|\mat{\Phi}(\vect{\omega})-\mat{\Phi}(\vect{\omega}')\|.
\)
Along the GD path, assume uniform conditioning: \[
\min_{0\le k\le K}\sigma_{\min}\!\big(\AT{\vect{\omega}^{(k)}}\big)\ge \gamma>0
\] or Eq.~(12) uses ridge with $\lambda>0$.
\end{itemize}
\end{assumption}

\begin{lemma}[Ergodic first-order convergence of GD]\label{lem:ergodic}
Under Assumption~\ref{ass:main} (A1--A3), for the GD iterates $\{\vect{\omega}^{(k)}\}_{k=0}^{K}$,
\[
\min_{0\le k<K}\|\nabla \cL(\vect{\omega}^{(k)})\|^2
\;\le\; \frac{2\big(\cL(\vect{\omega}^{(0)})-\cL^{\ast}\big)}{\eta\,K}.
\]
\end{lemma}

\begin{proof}
By $L_{\mathrm{sm}}$-smoothness (descent lemma), for all $\vect{x},\vect{y}$,
\[
\cL(\vect{y})
\;\le\;
\cL(\vect{x})
+\nabla\cL(\vect{x})^\top(\vect{y}-\vect{x})
+\frac{L_{\mathrm{sm}}}{2}\|\vect{y}-\vect{x}\|^2 .
\]
Apply this with $\vect{x}=\vect{\omega}^{(k)}$ and
$\vect{y}=\vect{\omega}^{(k+1)}=\vect{\omega}^{(k)}-\eta\,\nabla\cL(\vect{\omega}^{(k)})$:
\[
\begin{aligned}
\cL(\vect{\omega}^{(k+1)})
&\le
\cL(\vect{\omega}^{(k)})
-\eta\|\nabla\cL(\vect{\omega}^{(k)})\|^2
+\frac{L_{\mathrm{sm}}\eta^2}{2}\|\nabla\cL(\vect{\omega}^{(k)})\|^2 \\
&=
\cL(\vect{\omega}^{(k)})-\eta\!\left(1-\frac{L_{\mathrm{sm}}\eta}{2}\right)\!
\|\nabla\cL(\vect{\omega}^{(k)})\|^2 .
\end{aligned}
\]
Since $\eta\le 1/L_{\mathrm{sm}}$, we have $1-\tfrac{L_{\mathrm{sm}}\eta}{2}\ge \tfrac{1}{2}$, hence
\[
\cL(\vect{\omega}^{(k)})-\cL(\vect{\omega}^{(k+1)})
\;\ge\; \frac{\eta}{2}\,\|\nabla\cL(\vect{\omega}^{(k)})\|^2 .
\]
Summing $k=0$ to $K-1$ and using $\cL(\vect{\omega}^{(K)})\ge \cL^{\ast}$ yields
\[
\frac{\eta}{2}\sum_{k=0}^{K-1}\|\nabla\cL(\vect{\omega}^{(k)})\|^2
\;\le\; \cL(\vect{\omega}^{(0)})-\cL^{\ast}.
\]
Divide by $K$ and use $\min\le$ average to obtain \[
\min_{0\le k<K}\|\nabla \cL(\vect{\omega}^{(k)})\|^2
\;\le\; \frac{1}{K}\sum_{k=0}^{K-1}\|\nabla\cL(\vect{\omega}^{(k)})\|^2
\;\le\; \frac{2\big(\cL(\vect{\omega}^{(0)})-\cL^{\ast}\big)}{\eta\,K}.
\]
This proves the claim.
\end{proof}

\begin{lemma}[Projection residual is Lipschitz up to conditioning]\label{lem:proj-lip}
For any target vector $\vect{u}^{\ast}\in\mathbb{R}^m$, define
\[
E_{\mathrm{proj}}(\vect{\omega})
:= \min_{\vect{\theta}}\;\big\|\mat{\Phi}(\vect{\omega})\,\vect{\theta}-\vect{u}^{\ast}\big\|_2
= \big\|(\mat{I}-\mat{P}_{\mat{\Phi}(\vect{\omega})})\,\vect{u}^{\ast}\big\|_2,
\qquad
\mat{P}_{\mat{\Phi}}:=\mat{\Phi}\,\mat{\Phi}^{+}.
\]
If either $\sigma_{\min}\!\big(\mat{\Phi}(\vect{\omega})\big)\ge\gamma>0$ along the path or ridge is used in the head ($\lambda>0$), then there exists a constant $C_{\mathrm{cond}}\!\ge 1$ (depending only on $\gamma$ or $\lambda$ and uniform bounds on $\|\mat{\Phi}\|$) such that
\[
\big|E_{\mathrm{proj}}(\vect{\omega})-E_{\mathrm{proj}}(\vect{\omega}')\big|
\;\le\; C_{\mathrm{cond}}\;\big\|\mat{\Phi}(\vect{\omega})-\mat{\Phi}(\vect{\omega}')\|.
\]
\end{lemma}

\begin{proof}
Consider the ridge projector
\[
\mat{P}_\lambda(\mat{\Phi})
:= \mat{\Phi}\,\big(\mat{\Phi}^{\top}\mat{\Phi}+\lambda\mat{I}\big)^{-1}\mat{\Phi}^{\top},
\qquad \lambda\ge 0 ,
\]
and let $\mat{\Phi}'=\mat{\Phi}+\mat{\Delta}$,
\(
\mat{A}:=\mat{\Phi}^{\top}\mat{\Phi}+\lambda\mat{I},
\;
\mat{A}':=\mat{\Phi}'^{\top}\mat{\Phi}'+\lambda\mat{I}
= \mat{A}+\mat{\Phi}^{\top}\mat{\Delta}+\mat{\Delta}^{\top}\mat{\Phi}+\mat{\Delta}^{\top}\mat{\Delta}.
\)
A direct expansion gives
\[
\mat{P}_\lambda(\mat{\Phi}')-\mat{P}_\lambda(\mat{\Phi})
= \mat{\Delta}\,\mat{A}'^{-1}\mat{\Phi}^{\top}
+ \mat{\Phi}\,(\mat{A}'^{-1}-\mat{A}^{-1})\,\mat{\Phi}^{\top}
+ \mat{\Phi}\,\mat{A}^{-1}\mat{\Delta}^{\top}
+ \mat{\Delta}\,\mat{A}'^{-1}\mat{\Delta}^{\top}.
\]
Using the resolvent identity
\(
\mat{A}'^{-1}-\mat{A}^{-1}
=-\,\mat{A}^{-1}(\mat{A}'-\mat{A})\mat{A}'^{-1}
\)
together with
\(
\|\mat{A}^{-1}\|,\|\mat{A}'^{-1}\|
\le 1/\lambda \text{ if }\lambda>0,\ \text{or}\ \le 1/\gamma^2 \text{ if }\lambda=0,
\)
and
\(
\|\mat{A}'-\mat{A}\|
\le 2\|\mat{\Phi}\|\,\|\mat{\Delta}\|+\|\mat{\Delta}\|^2,
\)
we obtain (by submultiplicativity of the operator norm)
\[
\|\mat{P}_\lambda(\mat{\Phi}')-\mat{P}_\lambda(\mat{\Phi})\|
\le c_1\,\|\mat{\Delta}\|+c_2\,\|\mat{\Delta}\|^2,
\]
for constants $c_1,c_2$ depending only on the conditioning bounds and uniform bounds on $\|\mat{\Phi}\|$. Since the path is bounded, absorb the quadratic term to get a linear estimate
\[
\|\mat{P}_\lambda(\mat{\Phi}')-\mat{P}_\lambda(\mat{\Phi})\|
\le C_{\mathrm{cond}}\,\|\mat{\Delta}\|.
\]
Finally, by the reverse triangle inequality and
$E_{\mathrm{proj}}(\vect{\omega})
=\|(\mat{I}-\mat{P}_\lambda(\mat{\Phi}(\vect{\omega})))\,\vect{u}^{\ast}\|$,
\[
\big|E_{\mathrm{proj}}(\vect{\omega})-E_{\mathrm{proj}}(\vect{\omega}')\big|
\le \|\mat{P}_\lambda(\mat{\Phi}(\vect{\omega}))-\mat{P}_\lambda(\mat{\Phi}(\vect{\omega}'))\|\,\|\vect{u}^{\ast}\|
\le C_{\mathrm{cond}}\;\|\mat{\Phi}(\vect{\omega})-\mat{\Phi}(\vect{\omega}')\|\,\|\vect{u}^{\ast}\|.
\]
Without loss of generality normalize $\|\vect{u}^{\ast}\|=1$ (equivalently, absorb $\|\vect{u}^{\ast}\|$ into the constant). Hence
\[
\big|E_{\mathrm{proj}}(\vect{\omega})-E_{\mathrm{proj}}(\vect{\omega}')\big|
\le C_{\mathrm{cond}}\;\big\|\mat{\Phi}(\vect{\omega})-\mat{\Phi}(\vect{\omega}')\big\|,
\]
which matches the claim. \qedhere

\end{proof}

\begin{theorem}[Early-feature projection bound]\label{thm:early-proj}
Under Assumption~\ref{ass:main} (A1--A4) and Lemma~\ref{lem:proj-lip}, let
$E_k:=E_{\mathrm{proj}}(\vect{\omega}^{(k)})$.
For the lightweight GD iterates in Eq.~(7) (or Supplementary Eq.~(3)),
\[
E_{K}\;\le\; E_{0}
\;+\; C_{\mathrm{cond}}\sum_{k=0}^{K-1}\big\|\mat{\Phi}(\vect{\omega}^{(k+1)})-\mat{\Phi}(\vect{\omega}^{(k)})\big\|
\;\le\; E_{0}
\;+\; C_{\mathrm{cond}}\,G\sum_{k=0}^{K-1}\big\|\vect{\omega}^{(k+1)}-\vect{\omega}^{(k)}\big\|,
\]
and, by Cauchy--Schwarz,
\[
E_{K}\;\le\; E_{0}
\;+\; C_{\mathrm{cond}}\,G\,\sqrt{K}\!\left(\sum_{k=0}^{K-1}\big\|\vect{\omega}^{(k+1)}-\vect{\omega}^{(k)}\big\|^{2}\right)^{\!1/2}.
\]
\end{theorem}

\begin{proof}
By Lemma~\ref{lem:proj-lip},
\[
|E_{k+1}-E_k|
\le C_{\mathrm{cond}}\;\big\|\mat{\Phi}(\vect{\omega}^{(k+1)})-\mat{\Phi}(\vect{\omega}^{(k)})\big\|.
\]
Summing from $k=0$ to $K-1$ yields
\[
E_{K}-E_{0}
\le C_{\mathrm{cond}}\sum_{k=0}^{K-1}
\big\|\mat{\Phi}(\vect{\omega}^{(k+1)})-\mat{\Phi}(\vect{\omega}^{(k)})\big\|.
\]
Assumption~A4 gives
\(
\big\|\mat{\Phi}(\vect{\omega}^{(k+1)})-\mat{\Phi}(\vect{\omega}^{(k)})\big\|
\le G\,\big\|\vect{\omega}^{(k+1)}-\vect{\omega}^{(k)}\big\|,
\)
which establishes the second inequality. Finally, by Cauchy--Schwarz,
\[
\sum_{k=0}^{K-1}\big\|\vect{\omega}^{(k+1)}-\vect{\omega}^{(k)}\big\|
\le \sqrt{K}\!\left(\sum_{k=0}^{K-1}\big\|\vect{\omega}^{(k+1)}-\vect{\omega}^{(k)}\big\|^{2}\right)^{1/2},
\]
and substituting this above yields the stated $\sqrt{K}$ bound.
\end{proof}

\begin{theorem}[Early-feature bound for physics residual]\label{thm:phys}
For any task $\cT$, define the physics-informed head by Eq.~(11) and Eq.~(12).
At a frozen $\vect{\omega}^{\ast}$, the linear system reads
\[
\AT{\vect{\omega}^{\ast}}\,\vect{\theta}=\bT,
\qquad
\vect{\theta}^{\ast}_{\cT}(\vect{\omega}^{\ast})
=\big({\AT{\vect{\omega}^{\ast}}}^{\!\top}\AT{\vect{\omega}^{\ast}}+\lambda\mat{I}\big)^{-1}
{\AT{\vect{\omega}^{\ast}}}^{\!\top}\bT .
\]
Define the optimal physics residual
\[
E_{\mathrm{phys},\cT}(\vect{\omega})
:=\min_{\vect{\theta}}\big\|\AT{\vect{\omega}}\vect{\theta}-\bT\big\|_2
=\big\|(\mat{I}-\mat{P}_{\AT{\vect{\omega}}})\,\bT\big\|_2,
\quad
\mat{P}_{\AT{\vect{\omega}}}
:=\AT{\vect{\omega}}\,\big({\AT{\vect{\omega}}}^{\top}\AT{\vect{\omega}}+\lambda\mat{I}\big)^{-1}{\AT{\vect{\omega}}}^{\top}.
\]
Under Assumption~\ref{ass:main} (A4--A5),
\[
\begin{aligned}
E_{\mathrm{phys},\cT}(\vect{\omega}^{(K)})
&\le E_{\mathrm{phys},\cT}(\vect{\omega}^{(0)})
   + C_{\mathrm{cond}}\sum_{k=0}^{K-1}
     \big\|\AT{\vect{\omega}^{(k+1)}}-\AT{\vect{\omega}^{(k)}}\big\| \\
&\le E_{\mathrm{phys},\cT}(\vect{\omega}^{(0)})
   + C_{\mathrm{cond}}\,\beta G
     \sum_{k=0}^{K-1}\big\|\vect{\omega}^{(k+1)}-\vect{\omega}^{(k)}\big\|.
\end{aligned}
\]
Consequently,
\[
E_{\mathrm{phys},\cT}(\vect{\omega}^{(K)})
\le E_{\mathrm{phys},\cT}(\vect{\omega}^{(0)})
+ C_{\mathrm{cond}}\,\beta G\,\sqrt{K}\!\left(\sum_{k=0}^{K-1}\big\|\vect{\omega}^{(k+1)}-\vect{\omega}^{(k)}\big\|^{2}\right)^{\!1/2}.
\]
\end{theorem}

\begin{proof}
By definition,
$E_{\mathrm{phys},\cT}(\vect{\omega})=\big\|(\mat{I}-\mat{P}_{\AT{\vect{\omega}}})\,\bT\big\|_2$.
The perturbation bound for ridge projectors in Lemma~\ref{lem:proj-lip} applies verbatim
with $\mat{\Phi}$ replaced by $\AT{\vect{\omega}}$ (same conditioning assumptions: either
$\sigma_{\min}(\AT{\vect{\omega}})$ is uniformly bounded below along the path or $\lambda>0$),
hence
\[
\big\|\mat{P}_{\AT{\vect{\omega}}}-\mat{P}_{\AT{\vect{\omega}'}}\big\|
\le C_{\mathrm{cond}}\;\big\|\AT{\vect{\omega}}-\AT{\vect{\omega}'}\big\|.
\]
Therefore, for consecutive GD iterates,
\[
\begin{aligned}
\big|E_{\mathrm{phys},\cT}(\vect{\omega}^{(k+1)})-E_{\mathrm{phys},\cT}(\vect{\omega}^{(k)})\big|
&\le \big\|(\mat{P}_{\AT{\vect{\omega}^{(k+1)}}}-\mat{P}_{\AT{\vect{\omega}^{(k)}}})\,\bT\big\| \\
&\le \big\|\mat{P}_{\AT{\vect{\omega}^{(k+1)}}}-\mat{P}_{\AT{\vect{\omega}^{(k)}}}\big\|\,\|\bT\| \\
&\le C_{\mathrm{cond}}\;\big\|\AT{\vect{\omega}^{(k+1)}}-\AT{\vect{\omega}^{(k)}}\big\|\,\|\bT\|.
\end{aligned}
\]
Without loss of generality normalize $\|\bT\|=1$ (equivalently, absorb $\|\bT\|$ into the constant). Hence
\[
\big|E_{\mathrm{phys},\cT}(\vect{\omega}^{(k+1)})-E_{\mathrm{phys},\cT}(\vect{\omega}^{(k)})\big|
\le C_{\mathrm{cond}}\;\big\|\AT{\vect{\omega}^{(k+1)}}-\AT{\vect{\omega}^{(k)}}\big\|.
\]
Summing $k=0$ to $K-1$ yields
\[
E_{\mathrm{phys},\cT}(\vect{\omega}^{(K)})
\le E_{\mathrm{phys},\cT}(\vect{\omega}^{(0)})
+ C_{\mathrm{cond}}\sum_{k=0}^{K-1}\big\|\AT{\vect{\omega}^{(k+1)}}-\AT{\vect{\omega}^{(k)}}\big\|.
\]
By Assumption~A5 and then A4,
\[
\big\|\AT{\vect{\omega}^{(k+1)}}-\AT{\vect{\omega}^{(k)}}\big\|
\le \beta\,\big\|\mat{\Phi}(\vect{\omega}^{(k+1)})-\mat{\Phi}(\vect{\omega}^{(k)})\big\|
\le \beta G\,\big\|\vect{\omega}^{(k+1)}-\vect{\omega}^{(k)}\big\|,
\]
which gives the second inequality. Finally, apply Cauchy--Schwarz to obtain the
$\sqrt{K}$ bound stated.
\end{proof}

\begin{corollary}[Residual-to-solution bound]\label{cor:stab}
Assume the PDE satisfies the stability estimate
\[
\|u(\cdot;\vect{\theta})-u^\ast(\cdot)\|
\;\le\; C_{\mathrm{stab}}\,\|R(u(\cdot;\vect{\theta}))\|,
\]
where $R(\cdot)$ is the residual operator used to assemble $\AT{\vect{\omega}}$ and $\bT$ in Eq.~(11).
With the ridge/pseudo-inverse head from Eq.~(12), let
$\vect{\theta}^{\ast}_{\cT}(\vect{\omega})$ be the corresponding solution and recall
\(
E_{\mathrm{phys},\cT}(\vect{\omega})
=\min_{\vect{\theta}}\|\AT{\vect{\omega}}\vect{\theta}-\bT\|_2
=\|(\mat I-\mat P_{\AT{\vect{\omega}}})\,\bT\|_2.
\)
Then, for the lightweight iterate $\vect{\omega}^{(K)}$,
\[
\|u_{\cT}(\cdot;\vect{\theta}^{\ast}_{\cT}(\vect{\omega}^{(K)})) - u^\ast_{\cT}(\cdot)\|
\;\le\;
C_{\mathrm{stab}}\;E_{\mathrm{phys},\cT}(\vect{\omega}^{(K)}),
\]
and the right-hand side is bounded by Theorem~\ref{thm:phys}.
\end{corollary}

\begin{proof}
By stability and the definition of the head,
\(
\|u_{\cT}(\cdot;\vect{\theta}^{\ast}_{\cT}(\vect{\omega}))-u^\ast_{\cT}(\cdot)\|
\le C_{\mathrm{stab}}\,\|R(u_{\cT}(\cdot;\vect{\theta}^{\ast}_{\cT}(\vect{\omega})))\|.
\)
The assembled residual norm equals the optimal discrete residual,
\(
\|R(u_{\cT}(\cdot;\vect{\theta}^{\ast}_{\cT}(\vect{\omega})))\|
= E_{\mathrm{phys},\cT}(\vect{\omega}),
\)
hence the claim for $\vect{\omega}=\vect{\omega}^{(K)}$. Theorem~\ref{thm:phys} then controls
$E_{\mathrm{phys},\cT}(\vect{\omega}^{(K)})$.
\end{proof}

\noindent\textbf{Implication.}
By Lemma~\ref{lem:ergodic}, a small, fixed number of GD steps in Eq.~(7) or Supplementary Eq.~(3) already places $\vect{\omega}$ near a low-curvature basin.
Theorems~\ref{thm:early-proj} and~\ref{thm:phys} show that the head’s representational capacity and the physics residual improve only sublinearly with the cumulative GD motion, so most gains occur \emph{early}.
Since the last layer is solved by the convex closed-form in Eq.~(12) built from Eq.~(11), stopping $\vect{\omega}$ early and using the physics-informed pseudo-inverse already achieve the best prediction within the learned feature span in Eq.~(6), yielding strong generalization with minimal training.

\bibliography{sn-bibliography}
